\documentclass[twoside,11pt]{article}

\usepackage[preprint]{jmlr2e}
\usepackage[T1]{fontenc}
\usepackage{comment}
\usepackage{booktabs}

\usepackage{amsmath,amssymb,array,tabularx}
\usepackage[table,dvipsnames]{xcolor}
\usepackage{tikz}
\usetikzlibrary{arrows.meta}
\usepackage{enumitem}
\usepackage{rotating}
\usepackage{thmtools}
\usepackage{thm-restate}
\usepackage{tabularx}
\usepackage{makecell}

\usepackage{amsmath,amsfonts,bm,mathtools}

\def\eqref#1{equation~\ref{#1}}

\def\1{\bm{1}}

\def\vx{{\bm{x}}}

\DeclareMathAlphabet{\mathsfit}{\encodingdefault}{\sfdefault}{m}{sl}
\SetMathAlphabet{\mathsfit}{bold}{\encodingdefault}{\sfdefault}{bx}{n}

\def\gA{{\mathcal{A}}}

\def\gH{{\mathcal{H}}}

\def\gP{{\mathcal{P}}}

\def\gY{{\mathcal{Y}}}

\newcommand{\E}{\mathbb{E}}

\DeclareMathOperator*{\argmin}{arg\,min}

\usepackage{xcolor}

\definecolor{BrightOlive}{RGB}{0,150,20}

\renewcommand{\E}{\mathbb{E}}

\newcommand{\mystrut}[1][0.9em]{%
  \vrule width 0pt height 0pt depth #1\relax
}

\newcommand{\cl}[1]{\begin{tabular}[c]{@{}l@{}}#1\end{tabular}}  
\newcommand{\sub}[1]{\\[-1pt]{\scriptsize #1}}                  
\newif\ifinlineproofs
\inlineproofsfalse    

\ifinlineproofs
  
\else
  \excludecomment{inlineproof}
\fi
\newif\ifshowpolishlater
\showpolishlaterfalse      

\ifshowpolishlater
  \includecomment{polishlater}
  \newcommand{\polishlatertext}[1]{{\color{gray}#1}}   
\else
  \excludecomment{polishlater}
  \newcommand{\polishlatertext}[1]{}                   
\fi

\usepackage{lastpage}
\jmlrheading{xx}{2026}{1-\pageref{LastPage}}{6/26; Revised 9/26}{12/26}{21-0000}{Alamri, Caprio, and Brown}

\ShortHeadings{Subjective Risk Decomposition:~ \\A New View for Uncertainty Quantification}{Alamri, Caprio, and Brown}
\firstpageno{1}

\begin{document}

\title{
Subjective Risk Decomposition: \\A New View for Uncertainty Quantification}

\author{\name Raghad Alamri \email raghad.alamri@manchester.ac.uk \\
      \addr Department of Computer Science, The University of Manchester
      \AND
      \name Michele Caprio \email Michele.Caprio@warwick.ac.uk  \\
      \addr Department of Computer Science, University of Warwick
      \AND
      \name Gavin Brown \email gavin.brown@manchester.ac.uk \\
      \addr Department of Computer Science, The University of Manchester
      }

\editor{Our Editor}

\maketitle

\begin{abstract}
We present a new view of uncertainty quantification (UQ), in which uncertainty measures are not {\em primitives} to be defined
and debated,
but instead {\em consequences} of  modelling and evaluation choices.
We argue that {\em subjective risk}  provides a natural route to unify and extend the field. %
Given a chosen model and strictly proper loss, the corresponding
{\em subjective} risk admits decompositions where the components have meaningful interpretations for UQ.
We first show how familiar measures of epistemic/aleatoric uncertainty emerge from the expected subjective risk, such as Gal’s (2016) information-theoretic terms and other more recent proposals.
Next, we show how the geometry induced by the proper loss explains how entropy-based and variance-based measures are rooted in the same principle. 
Another decomposition allows us to connect UQ to statistical learning theory---showing how predictive uncertainty can change with the data, model capacity, and computation available to the learner.
Altogether, this provides both a unifying perspective on previous literature and a first step towards a full learning-theoretic framework for uncertainty quantification.
\end{abstract}

\begin{keywords}
  \em uncertainty, epistemic, aleatoric, subjective risk,
  mutual information.
\end{keywords}

\section{Introduction}
A popular framework for uncertainty quantification is the information-theoretic decomposition of \citet{gal2016uncertainty}.
This adopts a Bayesian view: define a model distribution $q_\Theta$, parametrised by $\Theta\sim \pi(\cdot\mid d_n)$, where $d_n=\{(\vx_i,y_i)\}_{i=1}^n$ denotes observed data.
Label predictions are obtained by marginalising parameters and sampling from the posterior predictive distribution, i.e. $\hat Y \mid \vx,d_n \sim \bar q(\cdot \mid \vx,d_n)$, where  $\bar q(y \mid \vx,d_n)=\mathbb E_{\Theta \sim \pi(\cdot \mid d_n)}\left[q_\Theta(y \mid \vx)\right]$.
The entropy of this predictive distribution admits a decomposition,
\begin{equation}
H(\hat{Y}\mid \vx, d_n)
=
I(\hat{Y};\Theta\mid \vx, d_n)
+
H(\hat{Y}\mid \Theta,\vx,d_n),
\label{eq:GalEntropy}
\end{equation}
whose terms are widely interpreted as uncertainties that are {\em reducible} (or `epistemic', literally, `relating to knowledge') and {\em irreducible} (or `aleatoric', literally, `relating to chance'). The decomposition itself follows directly from the chain rule of entropy, but it has rapidly become one of the most influential frameworks in the field.
The success of this formulation inspired a large body of follow-up work. Researchers proposed alternative measures of `epistemic' and `aleatoric' uncertainty, motivated by different modelling assumptions and application domains.
As a result, the literature now contains many competing definitions of uncertainty, often motivated from quite different perspectives.

One response to this proliferation has been 
to define and justify {\em axioms} for UQ terms \citep{pmlr-v216-wimmer23a, bulte2026}, sparking  community discussion \citep{kirsch2024twitter}.
Others have questioned whether the aleatoric/epistemic dichotomy {\em itself} is even meaningful 
\citep{kirchhof2025reexamining}.
\citet{smith2024rethinking} follow this, stating that 
the aleatoric/epistemic dichotomy is insufficiently expressive. They propose a decision-theoretic perspective based on subjective utility theory \citep{savagebook}---they define uncertainty as the subjective risk of acting Bayes-optimally, and study its expected reduction under Bayesian updating.
\citet{sale2026meaningful} take the alternative position, that the aleatoric/epistemic dichotomy is alive and well, and {\em ``that alleged failures} [of the dichotomy] {\em are better understood as flaws of particular mathematical frameworks''}.

We offer a new view on this debate---a new mathematical framework that provides both a unifying perspective and extends the field in new directions.
Like \citet{smith2024rethinking}, we adopt the subjective risk, but extend its scope beyond the Bayes-optimal action.  We treat subjective risk as a stochastic quantity influenced by training data, procedural randomness, and model parameters.
Assuming the expected subjective risk is based on a strictly proper loss, we show it admits a decomposition into three terms: external grounding, predictive disagreement, and average generalised entropy. The {\em grounding} term is a loss-dependent discrepancy between the model and reality, and is generally an oracle quantity; the remaining terms depend only on model predictions and require no ground-truth labels. 
Under the log-loss with categorical outcomes, this decomposition recovers the
classic information-theoretic uncertainty measures of \citet{gal2016uncertainty}.
With other distributional/loss assumptions, we recover numerous popular
measures previously motivated from
diverse perspectives \citep{Balaji2017, malinin_2019, salelabel}. We further show that the commonly held distinction between {\em entropy-} and {\em variance}-based approaches \citep{sale2023uncertainty,bulte2026} can be understood via a Bregman law of total variance, with differences explained by the geometry of the loss.
We then connect our framework to statistical learning theory \citep{bach2024learning}. We introduce the {\em subjective risk gap} and decompose it into analogues of optimisation, estimation, and approximation errors, each with grounding and uncertainty terms. 
Amongst other observations, this recovers  the {\em expected uncertainty reduction} proposed by \citet{smith2024rethinking} as a component of our {\em estimation gap}. 

In summary,
we show that existing notions of {\em uncertainty} need not be treated as  primitives requiring design and debate: they can instead be derived as {\em consequences} of higher-level modelling and evaluation choices. This approach naturally leads to an interpretation of uncertainty in a statistical learning theory framework. Our headline contributions are:

\begin{itemize}
  \setlength\itemsep{0.25em}

\item \textbf{A unified derivation.}
We show how several established uncertainty measures emerge from the decomposition of expected subjective risk under strictly proper losses, alongside an external grounding term (Section 4).

\item \textbf{An explanation of differences.}
We show how loss geometry explains the behaviour of different uncertainty measures, unifying entropy- and variance-based
approaches under appropriate assumptions (Sections 4.2 and 4.3).

\item \textbf{A bridge to learning theory.}
We introduce subjective analogues of optimisation, estimation, and approximation errors, providing initial foundations for a learning-theoretic understanding of uncertainty quantification (Section 5).

\end{itemize}

\newpage
\begin{table}[t]
\centering
\setlength{\tabcolsep}{18pt}
\renewcommand{\arraystretch}{1.4}
\caption{Summary of notation.}
\label{tab:notation}
\small
\begin{tabular}{@{}ll@{}}
\toprule
Notation & Meaning \\
\midrule
$\mathcal{X},\mathcal{Y}$ & Input and outcome spaces. \\
$X,Y$ & Input and true outcome random variables. \\
$p(\cdot\mid \vx)$ & True conditional distribution at input $\vx$. \\
$q(\cdot\mid \vx)$ & Model belief (predictive distribution) at input $\vx$. \\
$z(\cdot\mid \vx)$ & Generic reported distribution. \\
$\mathcal{P}$ & Convex ambient class of distributions. \\
$\mathcal{H}\subseteq\mathcal{P}$ & Model/hypothesis class. \\
$\hat Y\sim q(\cdot\mid \vx)$ & Outcome generated under the model belief. \\[1mm]
$\ell(z,y)$ & Loss of report $z$ when outcome $y$ occurs. \\
$L_p(z)$ & Pointwise objective risk,
$\mathbb{E}_{Y\sim p(\cdot\mid \vx)}[\ell(z,Y)]$. \\
$L_q(z)$ & Pointwise subjective risk,
$\mathbb{E}_{\hat Y\sim q(\cdot\mid \vx)}[\ell(z,\hat Y)]$. \\
$R_p(z)$ & Objective risk integrated over $X\sim p(\vx)$. \\
$R_q(z)$ & Subjective risk integrated over $X\sim p(\vx)$. \\
$H_\ell(q)$ & Generalised entropy (Bayes risk), $L_q(q)$. \\
$\phi$ & Convex generator associated with the proper loss. \\
$B_\phi(p,q)$ & Bregman divergence generated by $\phi$. \\[1mm]
$\Gamma:\mathcal{P}\to\mathcal{A}$ & Elicited target functional. \\
$a_q=\Gamma(q)$ & Action elicited by distribution $q$. \\
$t(y)=\Gamma(\delta_y)$ & Action elicited by certainty in outcome $y$. \\[1mm]
$D_n,\ d_n$ & Random and realised training datasets of size $n$. \\
$\Theta=T(D_n,A)$ & Random model parameters produced by training algorithm $T$. \\
$A$ & Procedural randomness in the learning procedure. \\
$q_\Theta$ & Random predictive distribution. \\
$\bar q=\mathbb{E}_\Theta[q_\Theta]$ & Aggregate predictive distribution (right centroid). \\
$V_\phi(a_\Theta)$ & Predictive disagreement, $\mathbb{E}_\Theta[B_\phi(a_\Theta,\bar a)]$. \\
$q_n$ & Empirical-risk-minimising predictor. \\
$q_*$ & Best-in-class predictor under objective population risk. \\
$\Delta_{\mathrm{obj}}(q)$ & Objective excess risk, $R_p(q)-R_p(p)$. \\
$\Delta_{\mathrm{subj}}(q)$ & Subjective risk gap, $R_q(p)-R_p(p)$. \\[1mm]
$G(z)$ & Grounding, $\mathbb{E}_X[B_\phi(z,p)]$. \\
$U(z)$ & Predictive uncertainty, $R_z(z)=\mathbb{E}_X[H_\ell(z)]$. \\[1mm]
$\mu_p,\mu_q,\mu_\Theta,\bar\mu$ & Predictive means when the loss elicits the mean. \\
\bottomrule
\end{tabular}
\end{table}
\newpage\clearpage

\setcounter{section}{1}
\section{Background: Proper Losses and Bregman Divergences}
\label{sec:background}

We review pertinent properties of proper losses and the associated Bregman divergences. %
For broader views see~\citet{gneiting2007strictly}, %
and~\citet{Nielsen2009}.

Define a convex class of probability distributions \(\mathcal{P}\) over an outcome space \(\mathcal{Y}\), all dominated by a common reference measure \(\nu\).
Throughout this section, we fix an input \(\vx\), and write \(p\) and \(q\) for conditional distributions over outcomes at that input.
Define a probabilistic loss, $\ell\colon \gP\times \gY\rightarrow \mathbb{R}\cup\{+\infty\}$,
where the first argument is a model distribution, and the second a realised outcome at which we evaluate the loss.
We assume $\ell(p,\cdot)$ is bounded below for every $p\in\mathcal{P}$, so that every loss discussed is well-defined.
The {\em pointwise risk} of a reported model $q\in\gP$  is its loss in expectation over the true $p\in\gP$, i.e.,

\newcommand{\objectiverisk}[1]{L_p(#1)}
\newcommand{\subjectiverisk}[1]{L_q(#1)}

\newcommand{\subjectiveriskqtheta}[1]{L_{q_\Theta}(#1)}

\begin{equation}
\objectiverisk{q} := \E_{Y\sim p(\cdot\mid\vx)}\big[\, \ell(q,Y) \,\big] = \int_{\gY} \ell(q,y)\,p(y\mid\vx)\, d\nu(y).
\label{eq:properrisk}
\end{equation}
\noindent The loss $\ell$ is {\em proper} 
iff for every $p\in\gP$, $\objectiverisk{q}$ is minimised at $q=p$,
and {\em strictly proper} if $q=p$ is the unique minimiser \citep{gneiting2007strictly}.
The canonical example of this is $\ell(q,y)=-\ln q(y\mid\vx)$, which induces $\objectiverisk{q}=-\int p(y\mid\vx)\ln q(y\mid\vx) \,dy$, the cross-entropy.
If $\gY$ is at most countable, the expectation is
with respect to the probability mass function. %
For example: if $p,q$ are categorical, $\nu$ is the counting measure, and $\ell(q,y)=-\ln q(y\mid\vx)$, we have $\objectiverisk{q} = -\sum_{y\in\gY}p(y\mid\vx)\ln q(y\mid\vx)$.
The minimum achievable risk is called the {\em Bayes} risk (or {\em generalised entropy}), defined as $H_\ell(p) := \inf_{z\in\gP} \objectiverisk{z} = \objectiverisk{p} =\int_\gY \ell(p,y)\,p(y\mid\vx)d\nu(y)$.

Following \citet{savage1971elicitation}, we use that strictly proper losses admit a representation in terms of a strictly convex function $\phi$ on $\gP$.
In this {\em Savage} representation, the {\em regret} of a strictly proper loss can be written as a Bregman divergence \citep{bregman1967relaxation},
\begin{eqnarray}\label{eq:properregretisbregman}
        \objectiverisk{q} - \objectiverisk{p} &=& \phi(p)-\phi(q)- \langle \nabla\phi(q), \, p-q \rangle ~=~ B_\phi(p,q),
\end{eqnarray} 
where the {\em Bregman generator} $\phi(p) := -\mathbb{E}_{Y\sim p(\cdot\mid\vx)}[\ell(p,Y)]$, and $\nabla\phi(q)$ denotes a subgradient of $\phi$ at $q$.
A simple rearrangement yields,
\begin{equation}\label{eq:divergenceplusentropy}
  \objectiverisk{q} = B_\phi(p,q) + H_\ell(p).
\end{equation}
That is, a strictly proper risk is a Bregman divergence plus a
generalised entropy\footnote{Adding an integrable outcome-only
term $c(y)$ to the loss leaves regret unchanged but shifts
$H_\ell(p)$ by $\E_{Y\sim p(\cdot\mid\vx)}[c(Y)]$.
The loss and reference measure are held fixed when comparing risks.
The additional convention that a correct certain report has zero
loss is imposed where such reports are admissible, as in
Theorem~\ref{thm:ent_var}.}.

When $\ell$ depends on $q$ only through distributional parameters, the Bregman geometry can be induced on a lower-dimensional parameter space rather than the full support of the distributions.
To formalise this, we adopt a decision-theoretic notation. Let $\mathcal{A}$ be a convex {\em action} space, and $\Gamma:\mathcal{P}\to\mathcal{A}$ be an {\em elicitable target} functional, mapping a distribution to the quantity being modelled (e.g. mean), referred to as the {\em action}.
Throughout the paper we write $a_p:=\Gamma(p)$ and $a_q:=\Gamma(q)$.
The loss is now in terms of an action, $\ell:\mathcal{A}\times\mathcal{Y}\to\mathbb{R}\cup\{+\infty\}$, and the pointwise risk $\objectiverisk{a_q}$ is the expected loss of that action over~$p$. 
The loss $\ell$ is {\em proper} for $\Gamma$ if, for every $p\in\gP$, $\Gamma(p) \in \argmin_{a\in\gA} \objectiverisk{a}$, and {\em strictly proper} if $\Gamma(p)$ is the unique minimiser.
A strictly proper loss therefore recovers $\Gamma(p)$ as the optimal reported action, with Bayes risk $H_\ell(p)=\objectiverisk{a_p}$ \citep{savage1971elicitation}.\footnote{In the point-forecasting literature, this property is called (strict) {\em consistency} \citep[Def.~1]{gneiting2011making}.}

Every strictly proper loss on the full distribution has the Bregman form of \autoref{eq:properregretisbregman} \citep[Thms.~1--2]{gneiting2007strictly}. For a general target this holds when $\Gamma$ is linear in $p$.\footnote{For example, it holds for the mean
\citep{banerjee2005optimality}, but not for the median
\citep{gneiting2011making}.}
In the remainder of the paper we therefore assume $\Gamma$ is linear in $p$, so that \autoref{eq:properregretisbregman} and \autoref{eq:divergenceplusentropy} hold on the action space, with $p,q$ in the divergence term replaced by $a_p,a_q$ and $\phi$ differentiable on $\mathrm{ri}(\gA)$.
At the elicited action $a_p$, this generator equals the negative generalised entropy up to a term affine in $p$ \citep{gneiting2007strictly}.

\begin{table}[t]
\centering
\small
\renewcommand{\arraystretch}{1}
\setlength{\tabcolsep}{6pt}
\caption{Examples of proper losses and their induced Bregman geometries.}
\label{tab:proper-losses}
\begin{tabular*}{\linewidth}{@{}l@{\extracolsep{\fill}}l l l@{}}
\toprule
\textbf{Distribution} & \textbf{Proper loss} $\ell(q,y)$ & \textbf{Generator} $\phi$ & \textbf{Divergence} $B_\phi(p,q)$\\
\midrule
\cl{Categorical\sub{log-loss}}
 & \cl{$-\ln q(y)$}
 & \cl{$\sum\nolimits_y q(y)\ln q(y)$}
 & \cl{$\sum\nolimits_y p(y)\ln\dfrac{p(y)}{q(y)}$}\\
\cl{Categorical\sub{Brier}}
 & \cl{$\lVert e_y-q\rVert^2$}
 & \cl{$\sum\nolimits_y q(y)^2$}
 & \cl{$\sum\nolimits_y\big(p(y)-q(y)\big)^2$}\\
\cl{Gaussian, known $\sigma^2$\sub{squared loss}}
 & \cl{$(y-\mu_q)^2$}
 & \cl{$\mu_q^2$}
 & \cl{$(\mu_p-\mu_q)^2$}\\
\cl{Gaussian, unknown $\sigma^2$\sub{log-loss}}
 & \cl{$\dfrac{(y-\mu_q)^2}{2\sigma_q^2}$\\ $+\dfrac12\ln(2\pi\sigma_q^2)$}
 & \cl{$-\dfrac12\ln(2\pi e\sigma_q^2)$}
 & \cl{$\ln\dfrac{\sigma_q}{\sigma_p}-\dfrac12$\\ $+\dfrac{\sigma_p^2+(\mu_p-\mu_q)^2}{2\sigma_q^2}$}\\
\cl{Poisson\sub{log-loss}}
 & \cl{$\lambda_q-y\ln\lambda_q$\\ $+\ln(y!)$}
 & \cl{$\lambda_q\ln\lambda_q-\lambda_q$}
 & \cl{$\lambda_p\ln\dfrac{\lambda_p}{\lambda_q}+\lambda_q-\lambda_p$}\\
\bottomrule
\end{tabular*}
\end{table}

\paragraph{Proper losses and uncertainty quantification.}
Proper losses have been a fruitful area of research for UQ.
The terms in 
the identity $L_p(q)=B_\phi(p,q)+H_\ell(p)$
have long been assigned uncertainty semantics. \citet{kull2015novel}
interpreted the divergence $B_\phi(p,q)$  as epistemic/reducible loss, and $H_\ell(p)$ as aleatoric/irreducible uncertainty.
The same identity allows
the entropy of the aggregate predictive
distribution $\bar q$ to decompose into average entropy and disagreement, i.e. $H_\ell(\bar q)=\E_\Theta[H_\ell(q_\Theta)] + \E_\Theta[B_\phi(q_\Theta, \bar q)]$,
and the components are again commonly assigned the `aleatoric' and `epistemic' semantics.
Furthermore, changing the loss yields different measures, motivating task-aware uncertainty  \citep{hofman2026uncertainty}.

These decompositions share a mathematical foundation but differ in interpretation. Several proposals take the conventional `forward' view, where predictions are
evaluated under the true distribution \citep{gruber2023properBiasVariance,lahlou2023,kotelevskii2025from}.
\citet{schweighofer2025on}
take the `reverse' view, giving reverse cross-entropy under log-loss. They study several ways to approximate the unknown true distribution with different predictive forms.
However by doing this, the distinction between objective and subjective risk becomes somewhat blurred.
\citet{smith2024rethinking} argue that `aleatoric/epistemic' as terminology is insufficient to capture what is truly irreducible/reducible. In any scenario, {\em reducibility depends on what may change}. They introduce the {\em expected uncertainty reduction} as the change in uncertainty when we provide more data.%

We connect the underlying common thread in all these approaches by treating {\em subjective
risk} as a stochastic quantity influenced by learning resources.
We will show how existing uncertainty measures can be derived as components of a single decomposition alongside a grounding term that links the model to reality. %

\newpage

\section{Outline of Our Proposal}
\label{sec:proposal}

Our proposal supports the idea that uncertainty measures are not {\em primitives} needing to be {defined} through axioms and argumentation.  
Instead, they are {\em consequences} of modelling choices, in particular they are decomposition terms---{\em derived} given the modelling scenario.
Overall, our methodology
explains the classic information-theoretic terms of \citet{gal2016uncertainty}, the {\em expected uncertainty reduction} of \citet{smith2024rethinking}, and numerous other well-known uncertainty measures, whilst providing a bridge to statistical learning theory. %

\subsection{Subjective vs Objective Risks}

Our starting point is the decision-theoretic notion of {\em subjective risk}, introduced by \citet{ramsey1926truth} and formalised by \citet{savagebook}. Rather than evaluating predictions by the true data-generating distribution, subjective risk evaluates them under the agent's own beliefs.
Let $\gP$ be a convex class of distributions, and $\gH\subseteq \gP$ be the agent's {\em hypothesis class}, not necessarily convex.
Given a model $q\in\gH$,
the {\em subjective risk} at $\vx$
is defined as follows.

\begin{definition}[Pointwise Subjective Risk]\label{def:subjectiverisk}
Assume a loss $\ell:\gP\times\gY\rightarrow \mathbb{R}\cup\{+\infty\}$.
Given a model belief $q\in\gH\subseteq \gP$, and a reported distribution $z\in\gP$, the {\em subjective} risk of $z$ at $\vx$ is defined,
\begin{equation} \label{eq:subjectiverisk}
    \subjectiverisk{z}
    =
    \mathbb{E}_{\hat{Y} \sim q(\cdot\mid\vx)}
    \left[
        \ell(z,\hat{Y})
    \right] := \int_\gY \ell(z,y)\,q(y\mid\vx)\, d\nu(y).
\end{equation} 
If $\ell$ is strictly proper, the unique minimiser over $\gP$ is $z=q$; this gives the generalised entropy $H_\ell(q) := \subjectiverisk{q}$, i.e. the agent is incentivised to truthfully report
its belief $q$.
\end{definition}
In the original literature of \citet{savagebook}, this  is known as {\em subjective utility}, but we adopt a loss minimisation viewpoint for alignment with the modern ML literature.
The {\em subjective} risk can be contrasted with the conventional {\em objective} risk, $\objectiverisk{z}:=\E_{Y \sim p(\cdot\mid\vx)}\left[ \ell(z,Y) \right]$.
We use the terminology `objective risk' to clearly disambiguate, but it is the conventional `risk' used to evaluate ML models.
Analogously to \autoref{eq:divergenceplusentropy}, the subjective (proper) risk can be written\footnote{For a $\Gamma$ linear in $p$ (Section~\ref{sec:background}), the same holds with each distribution replaced by its action $a_z=\Gamma(z)$, giving $\subjectiverisk{a_z}=B_\phi(a_q,a_z)+H_\ell(q)$. Throughout the paper, we present our results in the full-distribution form for ease of notation and clarify when we use actions instead.} as a Bregman divergence plus a generalised entropy term,
\begin{equation}
    \subjectiverisk{z} = B_\phi(q,z) + H_\ell(q).
    \label{eq:subjective_Bregman}
\end{equation}
It is informative to take a special case---for log-loss on categorical distributions, the subjective risk of the true distribution $p$ under a model belief $q$, at a point $\vx$, is,

\begin{eqnarray}\label{eq:revcrossentropy}
    \subjectiverisk{p} &=& -\sum_{y\in\gY}q(y\mid\vx)\ln p(y\mid\vx) ~~ = ~~ D_{\mathrm{KL}}(q\mid\mid p) + H(q).
\end{eqnarray}
Here we generate outcomes according to the model, then assess how surprising those outcomes are under reality.
This expression is the `reverse' cross-entropy, and the right-hand side is the `reverse' KL divergence, plus the Shannon entropy $H(q)$.
We note that the reverse cross-entropy was proposed as a measure of `total' uncertainty by \citet[Eq. 5]{schweighofer2025on}. This is a special case of %
\autoref{eq:subjective_Bregman}, which applies for arbitrary strictly proper losses.

\subsection{Why Subjective Risk Is a Natural Foundation for UQ}
\label{subsec:whysubjectiverisk}

An important question for UQ is: {``when can I trust this model’s prediction?''}.  \citet{smith2024rethinking} argue:

\begin{quote}
{\em \small ``Uncertainty alone is not a reliable indicator of whether we can trust a model. Some kind of external grounding is crucial for well-informed practical deployment.''}\\  \small\hfill\citep[p.~7]{smith2024rethinking}
\end{quote}
We argue that the {subjective risk} under strictly proper losses, appropriately decomposed into component parts, is well-aligned with answering this, and provides a natural foundation to understand uncertainty.\\

Subjective risk is not merely a convenient evaluation functional, but has a classical foundation in statistical decision theory. \citet{savagebook} showed that preferences satisfying his rationality postulates admit a representation in terms of expected utility under subjective beliefs. In this sense, evaluating consequences %
under your model's beliefs is a consequence of rational behaviour, rather than an arbitrary modelling choice. %
The divergence in \autoref{eq:subjective_Bregman}
quantifies the cost, under the model’s own beliefs, of departing from rational reporting.
However rational our models may be, as \citet{smith2024rethinking} state, it is crucial to be grounded in reality if we want trustworthy models.  This is where we shift from Savage's classical framework. %
We do not use subjective risk as a {\em decision-making} framework. %
Instead, we use it as an {\em evaluation functional} for the model against reality, $\E_{\hat{Y}\sim q(\cdot\mid\vx)}[\ell(p,\hat{Y})]$. This fills the gap between internal rationality and external grounding.\\

We can also make a more pragmatic argument, from the viewpoint of an ML practitioner. 
The motivation is clearest in the case of reverse cross-entropy and reverse KL-divergence, but the idea is more general.
$D_{\mathrm{KL}}(q\mid\mid p)$ strongly penalises mass assigned by $q$ where $p$ is small, becoming infinite when $q$ assigns positive mass where $p$ assigns zero mass. 
Thus, it penalises confident belief in outcomes that are poorly supported by reality---which can be interpreted as
measuring a form of predictive {\em hallucination}.  
In the LLM literature, there is growing empirical evidence that hallucinations and reverse KL are strongly linked \citep{agarwal2024onpolicy,gu2024minillm,cao2025on}.
Furthermore, reverse KL regularisation is a widely used technique in RLHF \citep{ziegler2020finetuninglanguagemodelshuman, xiong2023iterative}, forcing the learned policy to stay close to a reference policy. This again can be interpreted as avoiding low-probability tokens that may lead to hallucinated actions.
\citet{malinin-rkl-2019} show that 
using reverse KL as a training criterion for their {\em Prior Networks}, results in better estimates of uncertainty and detection of adversarial attacks---that is, inputs that are implausible under reality.
We will show how these {\em hallucination-sensitive} behaviours extend beyond log-loss, and are interesting for UQ.\\

The objective risk $\E_{{Y}\sim  p(\cdot\mid\vx)}[\ell(q,{Y})]$ addresses the question {\bf ``how correct is this model?''}.
The subjective risk $\E_{\hat{Y}\sim q(\cdot\mid\vx)}[\ell(p,\hat{Y})]$ instead addresses {\bf ``how realistic is this model?''}, assigning high\footnote{Depending on loss geometry, as we will see later.} penalties when 
the model predicts outcomes that \(p\) regards as {\em implausible}.

\newpage

\section{A Decomposition to Unify Previous Literature}
\label{sec:GDE}
The UQ literature contains many competing definitions of uncertainty, often motivated from quite different perspectives.
In this section we show how subjective risk can be  a unifying perspective, deriving many established measures as components of a decomposition.

\subsection{The Grounding/Disagreement/Entropy Decomposition}
Uncertainty measures have commonly been defined as primitives, using axiomatic frameworks to provide justification. Instead, we argue that decompositions of the subjective risk provide a principled way to derive them. \autoref{thm:gde_decomposition} shows a decomposition of the expected subjective risk into three terms, each with relevant semantics for UQ.

\begin{theorem}[Grounding/Disagreement/Entropy (GDE) Decomposition]\label{thm:gde_decomposition}
Define a model $q_\Theta$, where parameters are $\Theta=T(D_n,A)$, given $(D_n,A)\sim P_{D_n,A}$.  Here,  $T$ denotes the training algorithm, $D_n$ is a random variable over i.i.d. datasets of size $n$, and $A$ captures all other sources of randomness in the learning procedure, e.g. batch sampling, initial parameters, etc.
We take the expected subjective risk of $p$, at a point $\vx$, under the belief of $q_\Theta$. %
Using the Savage representation (with associated assumptions), this can be written,
\begin{eqnarray}
        \underbrace{\mystrut[1em]
        \E_\Theta\Big[ \subjectiveriskqtheta{p}\Big]
        }_{\textnormal{expected subjective risk}}
        &=&
        \underbrace{\mystrut[1em] \textcolor{blue}{B_{\phi}(\bar{q} \,,\,p)}}_{\textnormal{grounding}}
        ~+~
        \underbrace{\mystrut[1em] \textcolor{BrightOlive}{\E_\Theta\Big[ B_{\phi}(q_\Theta\,,\,\bar q) \Big]}}_{\textnormal{predictive disagreement}}
        ~+
        \underbrace{\mystrut[1em] \textcolor{red}{\E_\Theta\Big[ H_\ell(q_\Theta) \Big].} }_{\textnormal{average predictive entropy}}\notag
\end{eqnarray}
In the special case of log-loss, the disagreement is the mutual information $I(\hat Y;\Theta)$,
and the final term is the Shannon conditional entropy $H(\hat Y\mid\Theta)$.
\end{theorem}  

The \textcolor{blue}{\bf grounding} term is the aggregate model's {\em subjective regret} $B_\phi(\bar q,p) = L_{\bar q}(p) - L_{\bar q}({\bar q})$ at a query point $\vx$. This represents the external grounding advocated by \citet{smith2024rethinking}.
The \textcolor{BrightOlive}{\bf disagreement} measures how the distributions vary around $\bar q$, due to randomness in $\Theta$. 
The \textcolor{red}{\bf average predictive entropy} is the uncertainty of each $q_\Theta$ conditional on $\Theta$. The proof is given in Appendix~\ref{app:proof_gde}. 

For categorical outcomes $\gY=\{1,\ldots,K\}$, each
$q_\Theta(\cdot\mid\vx)$ is a vector of class probabilities.
At a fixed input $\vx$, randomness in $\Theta$ therefore induces
a \emph{second-order probability distribution}: a distribution
over categorical parameter vectors on the probability simplex
\citep{pmlr-v216-wimmer23a,sale2023uncertainty}.
This includes distributions specified directly over categorical
parameters, as well as those induced by random model weights
through the map $\Theta\mapsto q_\Theta(\cdot\mid\vx)$.
The mean of this second-order distribution is
$\bar q(\cdot\mid\vx)$, and predictive disagreement measures
dispersion around this mean in the geometry of the loss.

The significance of this result\footnote{The result can also be seen as a `reverse-argument' bias/variance decomposition---cf. the conventional `forward' decomposition \citep{pfau2025} induces a variance $\E_\Theta[B_\phi(\mathring{q},q_\Theta)]$, where $\mathring{q}$ is a {\em left Bregman centroid} \citep{Nielsen2009}, whilst this reverse orientation induces the {\em right} centroid $\bar q$.}
is most evident in the case of log-loss: the disagreement and entropy are {\em exactly} the terms proposed by \citet{gal2016uncertainty} as appropriate for quantifying `epistemic' and `aleatoric' uncertainty. 
Other proposals from the literature can be similarly recovered as instantiations of \autoref{thm:gde_decomposition}---see \autoref{tab:gde-examples}.  In this way, numerous established measures, motivated by different principles, are special cases of a single principle.

\begin{table}[ht!]
\centering
\caption{\small Examples of \autoref{thm:gde_decomposition} under different distributional/loss assumptions.
Overall, the {\bf expected subjective risk}
$=$ \textcolor{blue}{\bf grounding}
$+$ \textcolor{BrightOlive}{\bf disagreement}
$+$ \textcolor{red}{\bf entropy}.
The disagreement and entropy terms recover uncertainty measures previously proposed in the literature, each motivated from a different perspective. Here, they emerge naturally from a common decomposition of expected subjective risk, alongside an additional grounding term.}
\label{tab:gde-examples}
\small
\setlength{\tabcolsep}{4pt}
\renewcommand{\arraystretch}{1.4}
\footnotesize
\begin{tabular}{>{\raggedright\arraybackslash}p{0.15\textwidth}|
                >{\centering\arraybackslash}p{0.15\textwidth}
                >{\centering\arraybackslash}p{0.15\textwidth}
                >{\centering\arraybackslash}p{0.18\textwidth}|
                >{\raggedright\arraybackslash}p{0.29\textwidth}}
\toprule
\footnotesize\bf Distribution\newline \em Loss & \footnotesize\bf \textcolor{blue}{Grounding\newline $B_\phi(\bar q, p)$} & {\footnotesize\bf \textcolor{BrightOlive}{Disagreement\newline $\E_\Theta \big[B_\phi(q_\Theta,\bar q)\big]$}} & \footnotesize\bf \textcolor{red}{~~~Entropy\newline $\E_\Theta\big[H_\ell(q_\Theta)\big]$} & {\footnotesize\bf In literature} \\
\midrule
Categorical\newline \small \em Log-loss
& $D_{KL}(\bar q\Vert p)$
& $I(\hat Y;\Theta)$
& $H(\hat Y\mid\Theta)$
& \citet{gal2016uncertainty} \\

Bernoulli\newline \small \em  Squared loss
& $(\bar\mu-\mu_p)^2$
& $\mathbb{E}_{\Theta}\big[(\mu_{\Theta}-\bar\mu)^2\big]$
& $\mathbb{E}_{\Theta}\big[\mu_{\Theta}(1-\mu_{\Theta})\big]$
& \citet{salelabel} \\

Gaussian \newline \em Log-loss
& $D_{KL}(\bar q\Vert p)$
& $\mathbb{E}_{\Theta}\big[D_{\mathrm{KL}}(q_{\Theta}\Vert\bar q)\big]$
& $\frac{1}{2}\mathbb{E}_{\Theta}\big[\ln(2\pi e\sigma_{\Theta}^{2})\big]$
& \citet{depeweg}\newline\citet{malinin_2019} \\

Gaussian \newline \em Squared loss
& $(\bar\mu-\mu_p)^2$
& $\mathbb{E}_{\Theta}\big[(\mu_{\Theta}-\bar\mu)^2\big]$
& $\mathbb{E}_{\Theta}\big[\sigma_{\Theta}^{2}\big]$
& \scriptsize\citet{kendall2017}\newline\citet{Balaji2017}\newline\citet{bulte2026} \\
\bottomrule
\end{tabular}
\end{table}

These terms are often referred to as `epistemic' and `aleatoric'.  The appropriateness of this has however been questioned \citep{smith2024rethinking}, a position we also subscribe to.  A term is not epistemic/reducible by nature of its algebraic form.  The decomposition above makes no reference to learning, or changing the availability of data. As such, we avoid referring to $I(\hat Y;\Theta)$ as `epistemic', and will later extend our theory to take this into account.

\paragraph{Data vs Procedural Uncertainty.}
\citet{huang2021quantifying} decomposed a (squared) variance into terms they refer to as \emph{data} plus \emph{procedural} uncertainty---that is, predictive uncertainty due to lack of/randomness in the training data, versus that due to stochastic behaviours in the learning algorithm. We can again explain this within our subjective risk framework.

In Theorem~\ref{thm:gde_decomposition}, $\Theta = T(D_n, A)$ depends on two sources of randomness: the variable over datasets, $D_n$ and the procedural variable $A$. Since the disagreement $\E_\Theta[B_\phi(q_\Theta,\bar q)]$ is a variance over a joint variable, the Bregman law of total variance can split this up by sources.
\begin{theorem}[Bregman LTV for primal variance (Gupta et al., 2022, Lemma 3.2)]
\label{thm:bregman_ltv}
Using the setting of \autoref{thm:gde_decomposition}, let $\bar q_{\scriptscriptstyle A\mid D_n} := \E_{A\mid D_n}[q_\Theta]$, where $A,D_n$ are random variables.
\begin{equation}
  \underbrace{\mystrut[0.5em] \E_\Theta\Big[B_\phi\big(q_\Theta,\bar q\big)\Big]}_{\textnormal{predictive disagreement}}
  =
  \underbrace{\mystrut[0.5em] \E_{D_n}\Big[B_\phi\big(\bar q_{\scriptscriptstyle A\mid D_n},\bar q\big)\Big]}_{\textnormal{data uncertainty}}
  +
  \underbrace{\mystrut[0.5em] \E_{D_n}\Big[\E_{A\mid D_n}\big[B_\phi\big(q_\Theta,\bar q_{\scriptscriptstyle A\mid D_n}\big)\big]\Big]}_{\textnormal{procedural uncertainty}} .
  \label{eq:bregman_ltv_dp}
\end{equation}
\end{theorem}
The first term captures predictive uncertainty due to $D_n$, while the second captures that due to the training procedure, conditional on $D_n$.  For squared loss, \autoref{eq:bregman_ltv_dp} reduces to,
\begin{equation}
  \E_\Theta\big[(\mu_\Theta-\bar\mu)^2\big]
  = \E_{D_n}\big[(\bar\mu_{\scriptscriptstyle A\mid D_n}-\bar\mu)^2\big]
  + \E_{D_n}\Big[\E_{A\mid D_n}\big[(\mu_\Theta-\bar\mu_{\scriptscriptstyle A\mid D_n})^2\big]\Big],
\end{equation}
which is exactly the data/procedural decomposition of \citet{huang2021quantifying,huang2023}.
The same split holds for every row of \autoref{tab:gde-examples}, each in its own geometry.\footnote{\citet[Eq.~7]{jimenez_2026position} discuss this split for other losses, where the variance is around the left Bregman centroid rather than the right centroid $\bar q$; for squared loss the two orientations coincide.}

\subsection{How Entropy-based and Variance-based Measures Relate}
\label{subsec:ent_var}

Two families of decomposition dominate the literature. 
For classification, one is the entropy-based decomposition of \citet{gal2016uncertainty},
\begin{equation}
  H(\hat Y) = I(\hat Y; \Theta) + H(\hat Y \mid \Theta),
   \label{eq:gal_decomp}
\end{equation}
seen as the disagreement and entropy columns of the first row of \autoref{tab:gde-examples}.
For regression, it is the variance-based decomposition of
\citet{kendall2017},
\begin{equation}
  \mathrm{Var}(\hat Y) = \mathrm{Var}_\Theta(\mu_\Theta)
  + \E_\Theta\big[\mathrm{Var}(\hat Y \mid \Theta)\big].
  \label{eq:var_decomp}
\end{equation}
In the last row of \autoref{tab:gde-examples}, these two terms appear as $\E_\Theta[(\mu_\Theta-\bar\mu)^2]$ and $\E_\Theta[\sigma_\Theta^2]$.
The two are usually studied in parallel, each assessed against a set of desirable properties \citep{sale2023uncertainty,bulte2026}. 
We show that both are instances of a (Bregman) law of total variance, in two different geometries.
To show this, we shift to an `action' notation. We write
$a_q := \Gamma(q)$ for the action elicited by a belief $q$ (so $a_\Theta$ for
$q_\Theta$, and $\bar a$ for $\bar q$), and $t(y) := \Gamma(\delta_y)$ for the action elicited when $y$ has probability one. For example, $t(y) = e_y \in \mathbb{R}^{|\mathcal{Y}|}$ for a categorical label, and $t(y) = y$ for
a real-valued outcome with $\Gamma(q) = \mu_q$.
\begin{theorem}[Entropy and variance decompositions are the law of total variance]
\label{thm:ent_var}
Under the assumptions of \autoref{thm:gde_decomposition}, suppose
that the action space regret representation extends to each
point mass belief $\delta_y$, with
$t(y)=\Gamma(\delta_y)\in\gA$,
$\phi(t(y))<\infty$, and $\ell(t(y),y)=0$ for all $y\in\gY$.
Then the loss is a Bregman divergence from the realised outcome
action to the reported action,
$\ell(a,y)=B_\phi(t(y),a)$ for $a\in\operatorname{ri}(\gA)$,
and generalised entropy is the expected divergence of the
realised outcome action from its mean $a_q$:
\begin{equation}
  H_\ell(q) = \E_{\hat Y\sim q}\Big[B_\phi\big(t(\hat Y),a_q\big)\Big].
  \label{eq:entropyIsVariance}
\end{equation}
Moreover, with $Z=(\Theta,\hat Y)$ and $\hat Y\mid\Theta\sim q_\Theta$, the variance of the realised outcome satisfies the law of total variance
\begin{equation}
  \underbrace{\mystrut[1em] \E_Z\Big[B_\phi\big(t(\hat Y),\bar a\big)\Big]}_{\textnormal{total variance}}
  =
  \underbrace{\mystrut[1em] \E_\Theta\Big[B_\phi\big(a_\Theta,\bar a\big)\Big]}_{\textnormal{between-group variance}}
  +
  \underbrace{\mystrut[1em] \E_\Theta\Big[\E_{\hat Y\mid\Theta}\big[B_\phi\big(t(\hat Y),a_\Theta\big)\big]\Big]}_{\textnormal{within-group variance}},
  \label{eq:ltv_outcome}
\end{equation}
which is the disagreement and entropy part of \autoref{thm:gde_decomposition} in this special case:
\begin{equation}
  H_\ell(\bar q)
  =
  \underbrace{\mystrut[1em] \textcolor{BrightOlive}{\E_\Theta\Big[B_\phi\big(q_\Theta,\bar q\big)\Big]}}_{\textnormal{predictive disagreement}}
  +
  \underbrace{\mystrut[1em] \textcolor{red}{\E_\Theta\big[H_\ell(q_\Theta)\big]}}_{\textnormal{average predictive entropy}}.
  \label{eq:gde_as_ltv}
\end{equation}
Log-loss on categorical distributions recovers Equation~\ref{eq:gal_decomp}, and squared loss recovers Equation~\ref{eq:var_decomp}---the same identity in two geometries.
\end{theorem}
The proof and further details are given in Appendix~\ref{app:sec4proofs}. Under different model assumptions, the result depends on whether the realised outcome is itself a valid action. For continuous distributions, squared loss elicits the
mean, so the outcome $y$ is a valid action, which gives the classical law of total
variance (\autoref{tab:gde-examples}, row~4).
Under log-loss on Gaussians (row~3), however, a point mass has no density, so the realised outcome is not a valid action. The entropy term is then the differential entropy, which can be negative and so is not a variance.

\newpage
\subsection{How Loss Geometry Changes Uncertainty}
\label{sec:goodUncertaintyMeasure}

Theorem~\ref{thm:gde_decomposition} decomposes expected subjective risk
into grounding, predictive disagreement, and average predictive entropy.
Theorem~\ref{thm:ent_var} further shows that, under additional assumptions,
the disagreement and entropy terms are the two components of a law of
total variance. These results explain how the quantities relate, but
their numerical values depend on the chosen loss, even for the same
predictive distributions. We now examine how the curvature of the
associated Bregman generator determines the size of grounding and
disagreement, and their sensitivity to changes in predictive beliefs.  Throughout this section we
fix the input $\vx$, the predictive family, and the linear target $\Gamma$.
Write $a_\Theta=\Gamma(q_\Theta)$ and
$\bar a=\mathbb E_\Theta[a_\Theta]=\Gamma(\bar q)$. The two divergence
terms in Theorem~\ref{thm:gde_decomposition} are
\[
  B_\phi(\bar a,a_p)
  \quad\text{and}\quad
  V_\phi(a_\Theta)
  :=\mathbb E_\Theta[B_\phi(a_\Theta,\bar a)]
  =\mathbb E_\Theta[\phi(a_\Theta)]-\phi(\bar a).
\]
The first compares the aggregate prediction with reality; the second
compares predictions with one another. 

\paragraph{Curvature Determines Comparisons.}
For finite-dimensional actions, curvature comparisons yield exact bounds,
not only local approximations. The following proposition records the
relevant consequences of the Bregman representation. We denote by $\preceq$ the positive semidefinite order, that is, 
$A\preceq B$ means $v^\top A v\le v^\top B v$,
for every tangent vector~$v$.

\begin{proposition}[Comparing Grounding and Predictive Disagreement]
\label{prop:curvature-comparison}
Let $C$ be a relatively open convex subset of a finite dimensional affine
action space, and let $\phi,\psi$ be strictly convex, twice continuously
differentiable generators on $C$. Fix $0<m\le M<\infty$.
The following statements are equivalent:
\begin{enumerate}[label=(\roman*)]
\item For every $a\in C$,
\[
  m\nabla^2\psi(a)\preceq\nabla^2\phi(a)
  \preceq M\nabla^2\psi(a).
\]
\item For every $u,v\in C$,
\[
  mB_\psi(u,v)\le B_\phi(u,v)\le MB_\psi(u,v).
\]
\item For every finitely supported random action $a_\Theta$ in $C$,
\[
  mV_\psi(a_\Theta)\le V_\phi(a_\Theta)
  \le MV_\psi(a_\Theta).
\]
\end{enumerate}
Derivatives and matrix inequalities are restricted to the tangent space
of the affine hull of $C$. The disagreement bounds also hold for arbitrary
random actions in $C$ whenever their mean belongs to $C$ and the
expectations are finite.
\end{proposition}

Proposition~\ref{prop:curvature-comparison} shows that bounds on relative
curvature throughout $C$ give the same multiplicative bounds on both
grounding and predictive disagreement, provided the relevant actions
lie in $C$. The converse establishes that these curvature bounds are
also necessary for the disagreement bounds to hold uniformly over all
finitely supported distributions of actions.
\begin{proof}[Proof of Proposition \ref{prop:curvature-comparison}]
Taylor's formula with integral remainder gives
\begin{equation}
  B_\phi(u,v)
  =\int_0^1(1-s)(u-v)^\top
  \nabla^2\phi\big(v+s(u-v)\big)(u-v)\,ds.
  \label{eq:integrated-curvature}
\end{equation}
Thus (i) implies (ii). Taking expectations in (ii), with
$u=a_\Theta$ and $v=\bar a$, proves (iii) and its extension to general
random actions. Conversely, take equal masses at $a+hv$ and $a-hv$,
where $v$ is a tangent direction and $h$ is sufficiently small.
Then
$V_\phi(a_\Theta)=\tfrac12h^2v^\top\nabla^2\phi(a)v+o(h^2)$,
and likewise for $\psi$. Dividing (iii) by $h^2/2$ and letting $h\to0$ proves~(i).
\end{proof}

The proposition provides a criterion for comparing disagreement under
two losses independently of the particular predictive family. The bounds
$mV_\psi\le V_\phi\le MV_\psi$ hold for every distribution of actions
in $C$ if and only if the curvature of $\phi$ lies between $m$ and $M$,
times that of $\psi$ at every point and in every tangent direction in $C$. The same bounds compare
grounding when $\bar a,a_p\in C$. If both losses satisfy the assumptions of Theorem~\ref{thm:ent_var},
their generalised entropies are expected Bregman divergences from the
realised outcome action $t(\hat Y)$ to its mean $a_q$. Whenever the divergence
bounds apply to these pairs of actions, taking expectations gives the
same multiplicative bounds on the two generalised entropies.

The connection with translation invariance is equally direct. Consider shifting every predicted action by the same vector $c$, without
changing its probability weight. The mean shifts by $c$ as well, and we
require all shifted actions to remain in $C$. We call disagreement
\emph{translation invariant} if such a shift leaves its value unchanged, 
$V_\phi(a_\Theta+c)=V_\phi(a_\Theta)$. 
This equality holds for every finite collection of actions, every choice
of probability weights, and every admissible shift if and only if
$\phi$ is quadratic---equivalently, its Hessian is constant on $C$. To see necessity, translate
the two-point distribution used in the proof between any two interior
actions and let $h\to0$; their Hessians must coincide. Conversely, for
$\phi(a)=\tfrac12a^\top Qa+b^\top a+c$,
$V_\phi(a_\Theta)=\tfrac12\mathbb E_\Theta[
(a_\Theta-\bar a)^\top Q(a_\Theta-\bar a)]$, which is translation
invariant. Thus the translation requirement discussed in
Section~\ref{sec:discussion} selects quadratic geometry.

\paragraph{Log-Loss and Squared Loss.}
For Bernoulli predictions, write $q_\Theta=\operatorname{Ber}(U)$ and
$\mu=\mathbb E[U]$. Squared loss $\ell(u,y)=(u-y)^2$ has generator
$\psi(u)=u^2$ and disagreement $\operatorname{Var}(U)$, as in
Table~\ref{tab:gde-examples}. Log-loss has generator
$\phi(u)=u\ln u+(1-u)\ln(1-u)$ and curvature
$\phi''(u)=1/[u(1-u)]$. Hence, if
$U\in[\varepsilon,1-\varepsilon]$ almost surely, where
$0<\varepsilon<1/2$, Proposition~\ref{prop:curvature-comparison} gives
\begin{equation*}
  2\operatorname{Var}(U)
  \le I(\hat Y;\Theta)
  \le \frac{\operatorname{Var}(U)}{2\varepsilon(1-\varepsilon)}.
\end{equation*}
Here, $\hat Y\mid\Theta\sim q_\Theta$, and logarithms are natural.
The bounds extend to $\varepsilon$ and $1-\varepsilon$ by continuity.
The categorical Brier loss $\|e_y-q\|^2$ in
Table~\ref{tab:proper-losses} is twice this squared loss in the binary
case, and therefore yields $2\operatorname{Var}(U)$.

More locally, for any scalar generator $\phi\in C^3$ on a neighbourhood
of a compact interval $J\subset(0,1)$ containing $U$ almost surely,
Taylor's theorem gives
\begin{equation*}
  V_\phi(U)
  =\frac12\phi''(\mu)\operatorname{Var}(U)+\mathcal R,
  \qquad
  |\mathcal R|\le
  \frac{\sup_{v\in J}|\phi'''(v)|}{6}\,
  \mathbb E[|U-\mu|^3].
\end{equation*}
Indeed, expand $\phi(U)$ about $\mu$ and use
$\mathbb E[U-\mu]=0$ in the Jensen gap expression for $V_\phi$.
For log-loss, the leading term is
$\operatorname{Var}(U)/[2\mu(1-\mu)]$; for squared loss the remainder
vanishes. This explains the location dependence of mutual information
even at fixed spread. The approximation concerns curvature at $\mu$;
the corresponding local grounding expansion is around the true action
$a_p$. 

\paragraph{Boundary Sensitivity Depends on Accumulated Curvature.}
An unbounded Hessian does not by itself imply an unbounded penalty for
assigning mass to an outcome that reality regards as implausible.
For a twice continuously differentiable convex binary generator, fixed
$u\in(0,1)$, and $0<t<u$, \eqref{eq:integrated-curvature} becomes
\begin{equation*}
  B_\phi(u,t)=\int_t^u(u-s)\phi''(s) ds.
\end{equation*}
Since the integrand is non-negative and $u-s$ is bounded away from zero
near $s=0$, grounding diverges as $t\downarrow0$ if and only if
$\phi''$ is non-integrable near zero. For log-loss,
$\phi''(t)\sim1/t$, so
$D_{\mathrm{KL}}(\operatorname{Ber}(u)\Vert\operatorname{Ber}(t))\to\infty$.
In contrast, the pseudo-spherical generator
\[
  \phi_\beta(t)=\big(t^\beta+(1-t)^\beta\big)^{1/\beta},
  \qquad \beta=\tfrac32,
\]
has curvature
\[
  \phi_\beta''(t)
  =\frac{[t(1-t)]^{-1/2}}
  {2\big(t^{3/2}+(1-t)^{3/2}\big)^{4/3}}
  \sim\frac{1}{2\sqrt t},
\]
which diverges but is integrable.
Indeed, $\phi_\beta(t)\to1$ and $\phi_\beta'(t)\to-1$, giving
\[
  \lim_{t\downarrow0}B_{\phi_\beta}(u,t)
  =\phi_\beta(u)-1+u<\infty.
\]
Thus sensitivity to small predictive fluctuations and sensitivity to
unsupported outcomes are distinct consequences of the loss geometry.
These comparisons concern fixed loss normalisations. Multiplying a loss
by a positive constant scales grounding and disagreement by that
constant; adding an affine term to its generator leaves both unchanged.
These results turn desired properties of uncertainty measures into
restrictions on the loss. Requiring disagreement to remain unchanged
when all predictions are shifted together selects quadratic generators.
In the binary setting, however, quadratic generators cannot also yield
an unbounded grounding penalty when reality assigns vanishing probability
to an outcome receiving positive model probability. That behaviour
requires non-integrable boundary curvature. The choice of loss determines which of these requirements the resulting measures can satisfy, while the curvature comparison quantifies how disagreement changes
with the~loss.

\subsection{Summary}
In this section we showed how numerous existing uncertainty measures can be derived from a single principle---decomposing the expected subjective risk under a strictly proper loss. The decomposition exposes terms for {\em predictive disagreement} and {\em average entropy}, which correspond exactly to terms conventionally referred to as {\em epistemic} and {\em aleatoric} \citep{gal2016uncertainty}.
The decomposition also exposes a {\em grounding} term, an idea advocated by \citet{smith2024rethinking}.
This {\em grounding} is asymmetric: it asks whether 
outcomes that the model considers plausible are supported by reality, rather than whether the model covers every outcome that reality considers plausible.

It is important to note that all constructions in this section deal with {\em static} scenarios, i.e. how much predictive disagreement there is in a pre-trained model.   This does not take account of scenarios with {\em interventions}, for example adding more training data, or expanding the model capacity. For this, we need new machinery, borrowing concepts from statistical learning theory. This is addressed in the next section.

\newpage\clearpage
\section{A Bridge to Statistical Learning~Theory}
\label{sec:SLT}
Statistical learning theory (SLT)  provides a rigorous theoretical foundation for supervised learning \citep{bach2024learning}. %
It would be desirable for {\em uncertainty quantification} to enjoy a similarly strong theoretical foundation. 
In this section we pursue this, taking a first step towards a full learning-theoretic framework for uncertainty quantification.

\subsection{Resources Determine Reducibility}\label{subsec:resources}

The UQ community has dedicated substantial effort around the information-theoretic identity $H(\hat Y) = I(\hat Y;\Theta)+H(\hat Y\mid\Theta).$
This effort was sparked by \citet{gal2016uncertainty}, interpreting $H(\hat Y\mid\Theta)$ as `aleatoric', or irreducible uncertainty, and $I(\hat Y;\Theta)$ as `epistemic', or reducible uncertainty.
However, as pointed out by \citet{smith2024rethinking}, nothing in these expressions makes any reference to the {\em process of learning}, or the {\em availability of data}.
Thus, we cannot state unequivocally that $I(\hat Y;\Theta)$, or any other quantity commonly interpreted as `epistemic', would be reduced if we had more data.
Their solution is the {\em expected uncertainty reduction} (EUR), which explicitly specifies the change in uncertainty when additional training data is provided.

We agree with this view, but propose a generalisation: {\em a quantity can be regarded as `reducible' only with respect to the availability of a specified resource}. %
Some uncertainties may be reduced by additional training data. However there are other interventions, such as enlarging the model class, or increasing resources such that the model can be better optimised.
The question ``is this uncertainty reducible?'' is therefore incomplete; we must instead ask ``reducible \emph{by provision of which resource}?''.

Statistical learning theory \citep{bottou2007tradeoffs, bach2024learning} already offers a precise language for this: the decomposition of excess risk into \emph{optimisation}, \emph{estimation} and \emph{approximation} errors.
Denote the population risk as $R_p(q):=\E_{X\sim p(\vx)}[L_p(q) \mid X]$.  We can decompose the {\em excess population risk}, i.e. model risk above that of the true $p$.
\begin{equation}
\label{eq:obj_decomp}
    \Delta_{\mathrm{obj}}(q)
    :=
    \underbrace{\mystrut[1em] R_{p}(q) - R_p(p)}_{\text{excess~risk}}
    ~=~
    \underbrace{\mystrut[1em] R_{p}(q) - R_p({q_{n}})}_{\text{optimisation error}}
    + \underbrace{\mystrut[1em] R_p({q_{n}}) - R_p({q_*})}_{\text{estimation error}}
    + \underbrace{\mystrut[1em] R_p({q_*}) - R_p(p) }_{\text{approximation error}}
\end{equation}
where $q$ is an arbitrary model, $q_n$ is an {\em empirical risk minimiser} and $q_*$ a {\em best-in-class model} \citep{bach2024learning}.  The excess risk is by definition non-negative, as are the estimation/approximation errors, whilst the optimisation error can be negative under certain conditions.
Provision of more computing resource has the potential to reduce the optimisation error, through more time allowed to find the empirical risk minimiser. Provision of more data can reduce estimation error. Provision of a larger/more expressive model family may decrease approximation error, at the cost of increasing estimation/optimisation error. 

Overall, the point here is that SLT already has a `language' for discussing reduction of risk components.  %
Nevertheless, the uncertainty literature uses terms like ``model uncertainty’’ and ``approximation uncertainty’’ inconsistently, and often conflates approximation/estimation with bias/variance~\citep{brown2024notthesame}. We adopt the standard terms of SLT, as defined in \autoref{eq:obj_decomp}. The remainder of this section develops a subjective counterpart to each, with the aim of a well-defined terminology for uncertainty quantification.

\newpage
\subsection{The Subjective Risk Gap}

In Section~\ref{subsec:resources}, every risk is evaluated under the true distribution $p$---outcomes are drawn from reality, and we ask how well a given report $z$ predicts them. Under subjective risk, the same report $z$ is scored on outcomes drawn from a {\em belief} $q$.
We extend the pointwise definition to the whole input space, as follows.

\begin{definition}[Subjective Risk]
Given a belief $q \in \gP$, the subjective risk of a reported distribution $z$ %
integrated over $p(\vx)$, is
\begin{equation}
\label{eq:slt_decomp}
    R_q(z) :=
    \E_{X\sim p(\vx)}\Big[L_q(z)\Big]
    = \E_{X\sim p(\vx)}\Big[ \E_{\hat{Y}\sim q(\cdot\mid X)}\big[ \ell(z, \hat Y) \big] ~\Big\vert~ X\Big].
\end{equation}
\end{definition}

The excess risk $\Delta_{\mathrm{obj}}(q)$ compares two {\em reports} under the {\em same} outcome distribution. Its subjective counterpart keeps the report fixed at $p$ and compares two {\em distributions} that generate the outcomes.
\begin{definition}[Subjective Risk Gap]
\label{def:subjgap}
The subjective risk gap of a belief $q\in\gP$ is\footnote{We assume $\E_{X\sim p(\vx)}[H_\ell(q)]$ and $\E_{X\sim p(\vx)}[H_\ell(p)]$ are finite.}
\begin{equation}
\label{eq:subjgap}
\Delta_{\mathrm{subj}}(q) := R_q(p) - R_p(p)
= \E_{X\sim p(\vx)}\Big[\,\, \E_{\hat{Y}\sim q(\cdot\mid X)}\big[ \ell(p,\hat Y) \big] - \E_{{Y}\sim p(\cdot\mid X)}\big[ \ell(p, Y) \big] \,\,~\Big\vert~ X\Big].
\end{equation}
\end{definition}
The excess risk $\Delta_{\mathrm{obj}}(q)$ remains the quantity that learning aims to minimise. The subjective risk gap is not something we should minimise directly; we use it as a \emph{diagnostic measure} to track the model's uncertainty against the uncertainty that remains when nothing about the model can be improved.
Note that comparing uncertainties alone is not enough \citep{smith2024rethinking}---a model can have the same uncertainty as reality while predicting the wrong outcomes. \autoref{thm:subjgap} therefore decomposes the subjective risk gap into a non-negative  {\em grounding} term and a signed {\em uncertainty gap}. This answers two questions: ``How far is the model from reality?'' and ``Is the model more or less uncertain than reality?''.
\begin{theorem}[Subjective Risk Gap as Grounding plus Uncertainty Gap] \label{thm:subjgap}
~\\
The~subjective risk gap of $p$ under a model belief $q$ can be rewritten as,
\begin{equation}
\underbrace{\mystrut[1em] R_q(p)-R_p(p)}_{\textnormal{subjective risk gap}}
~=~ \E_{X\sim p(\vx)}\Big[\, \underbrace{\mystrut[1em] B_\phi(q,p)}_{\textnormal{grounding}}
~+~ \underbrace{\mystrut[1em] H_\ell(q)-H_\ell(p)}_{\textnormal{uncertainty gap}} \,~\Big\vert~ X\Big].
\label{eq:gap_decomp}
\end{equation}
In the special case of log-loss, the grounding is the reverse KL-divergence $D_{\mathrm{KL}}(q \| p)$, and the uncertainty gap is the difference in Shannon entropies $H(q)-H(p)$.
\end{theorem}
\begin{figure}[t]
\centering
\includegraphics[width=\textwidth]{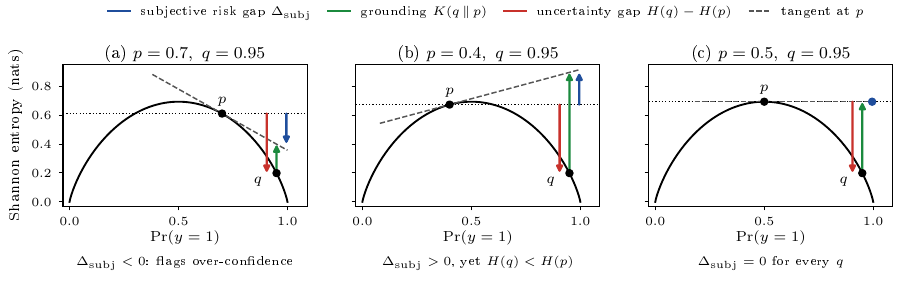}
\caption{The subjective risk gap for a Bernoulli outcome under log-loss. For a model $q$, the uncertainty gap and grounding sum to the subjective risk gap. In all three panels the model is over-confident, $H(q)<H(p)$, yet the gap is (a) negative, (b) positive, and (c) zero.}
\label{fig:subjgap_tangent}
\end{figure}
The {\em grounding term} is non-negative and, under strict propriety for the full distribution, vanishes iff $q=p$.\footnote{For the functional-eliciting examples, this vanishes iff $\Gamma(q)=\Gamma(p)$; under squared loss, this is equality of the means.} It is the {\em subjective regret} $R_q(p)-R_q(q)$: the cost under the model's own belief of reporting reality rather than its own prediction. The {\em uncertainty gap} compares the model's predictive uncertainty \(H_\ell(q)\) with the uncertainty of the true distribution \(H_\ell(p)\), averaged over \(p(\vx)\).  A negative value means the model is more confident than reality.
We now take a geometric view to see how the grounding and uncertainty gap interact. Since $\phi=-H_\ell$, the subjective risk gap can be written as\footnote{With $\phi=-H_\ell$, $B_\phi(q,p)=H_\ell(p)-H_\ell(q)+\langle\nabla H_\ell(p),q-p\rangle$, so adding the uncertainty gap $H_\ell(q)-H_\ell(p)$ leaves only the last term.}
\begin{equation}
\Delta_{\mathrm{subj}}(q)
=
\E_{X\sim p(\vx)}
\Big[
\big\langle
\nabla H_\ell(p),\, q-p
\big\rangle
~\Big\vert~ X
\Big].
\label{eq:gap_gradient}
\end{equation}
This characterises the sign of the gap through the {\em direction} in which the model deviates from reality. A zero subjective risk gap does not mean the model is correct. It only means that the grounding term and the uncertainty gap cancel. For the rest of the paper, we denote the two parts of $\Delta_{\mathrm{subj}}$ by the grounding term $G(z):=\mathbb{E}_{X\sim p(\boldsymbol{x})}[B_\phi(z,p)~\vert~ X]$ and the predictive uncertainty $U(z):=R_z(z)=\mathbb{E}_{X\sim p(\boldsymbol{x})}[H_\ell(z_X)~\vert~ X]$ of a belief $z$, so that $R_z(p)=G(z)+U(z)$ and $\Delta_{\mathrm{subj}}(z) = G(z) +U(z)-U(p)$.
\begin{proposition}[Zero gap]\label{prop:zero-gap}
 If reality is at an interior maximum of the entropy for almost every $X$, then $\nabla H_\ell(p)=0$ and the gap is zero for any model $q$.
\end{proposition}
\noindent\autoref{fig:subjgap_tangent} shows three over-confident models for which $\Delta_{\mathrm{subj}}(q)$ is negative, positive, or zero.
The sign of the gap alone therefore does not reveal over-confidence, and $\Delta_{\mathrm{subj}}(q)$ should always be read through its decomposition.

\paragraph{Diagnosing a model with the subjective risk gap.} Since the grounding is non-negative, the subjective risk gap is an upper bound on the uncertainty gap $U(q)-U(p)\le\Delta_{\mathrm{subj}}(q)$. A negative gap therefore always indicates over-confidence. A positive gap, however, does not imply under-confidence. Thus, we need to compare $\Delta_{\mathrm{subj}}(q)$ with $G(q)$, rather than with zero. Proposition~\ref{prop:diagnose} makes this precise.

\newpage
\begin{proposition}[Diagnosing the model through its uncertainty gap]
\label{prop:diagnose}
For any belief $q$ the uncertainty gap is $U(q)-U(p) = \Delta_{\mathrm{subj}}(q)-G(q)$. Exactly one of the following holds.
\begin{enumerate}[label=(\roman*)]
\item {\em Over-confident on likely outcomes.} If $\Delta_{\mathrm{subj}}(q)<0$, then $U(q)-U(p)\le\Delta_{\mathrm{subj}}(q)<0$; the model is therefore over-confident by at least $|\Delta_{\mathrm{subj}}(q)|$.

\item {\em Over-confident on unlikely outcomes.}  If $0\le\Delta_{\mathrm{subj}}(q)\le G(q)$, then the model is at most as uncertain as reality $U(q)\le U(p)$, and the grounding is at least as large as the over-confidence: $U(p)-U(q)\le G(q)$.
\item {\em Under-confident.} If $\Delta_{\mathrm{subj}}(q)>G(q)$, then $U(q)>U(p)$. That is, the model is
  more uncertain than reality by exactly $\Delta_{\mathrm{subj}}(q)-G(q)$.
\end{enumerate}
The same holds pointwise at each input $\vx$.
\end{proposition}
In the Bernoulli log-loss example with $p=0.7$ in
Figure~\ref{fig:subjgap_regions}, both $q=p$ and $q=1-p$ have
the same entropy as reality. Grounding is zero at $q=p$ and
positive at $q=1-p$. Thus equal predictive uncertainty does
not imply correct predictive beliefs.
\begin{figure}[t]
\centering
\includegraphics[width=0.70\textwidth]{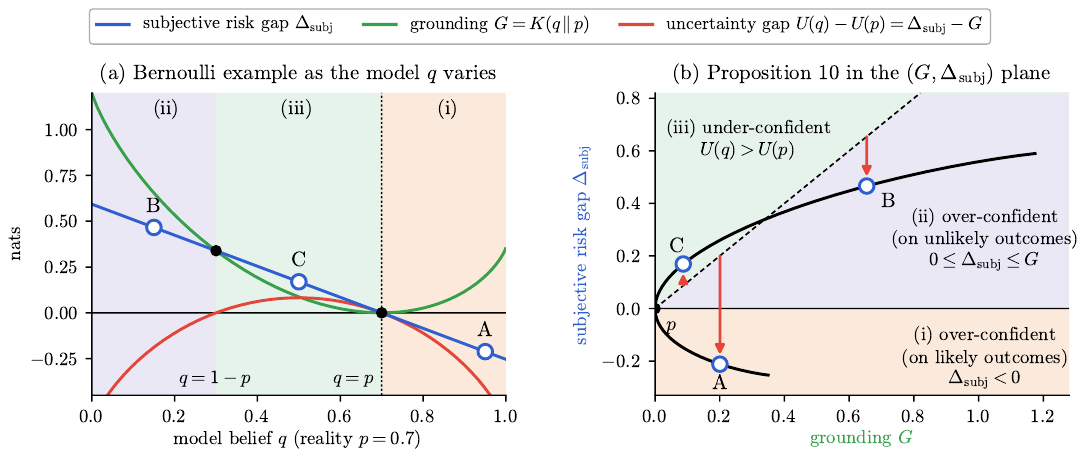}
\caption{The three cases of Proposition~\ref{prop:diagnose} for a Bernoulli outcome under log-loss, with reality $p=0.7$. (a) As $q$ varies, the cases change at $q=1-p$, where the uncertainty gap is zero, and at $q=p$, where the model matches reality. (b) The same models in the $(G,\Delta_{\mathrm{subj}})$ plane. The diagonal corresponds to $U(q)=U(p)$, and the vertical distance from it is the uncertainty gap $\Delta_{\mathrm{subj}}-G$. Models A, B, and C give one example from each case.}
\label{fig:subjgap_regions}
\end{figure}

\subsection{Grounding and Uncertainty under Changes of Resource}
\label{subsec:subjgaps}

In Section~\ref{subsec:resources}, the excess risk was decomposed into optimisation, estimation and approximation errors (\autoref{eq:obj_decomp}).
These errors were defined in terms of the {\em objective risk} $R_p(q)$.
We can now introduce corresponding notions induced by the {\em subjective risk}, $R_q(p)$, and establish basic properties.
We first consider the subjective risk gap of $p$, under a {\em fixed} model belief $q$. 
Using the {\em same} models $q_n$ and $q_*$\footnote{For the remainder of this section, we assume the best-in-class $q_*$ is unique.}, we can write it as a sum of (signed) components,

\begin{equation}
\label{eq:subj_decomp}
    \underbrace{\mystrut[0.5em] R_{q}(p) - R_p(p)}_{\text{subjective~risk~gap}}
    ~=~
    \underbrace{\mystrut[0.5em] R_{q}(p) - R_{q_{n}}(p)}_{\text{optimisation gap}}
    + \underbrace{\mystrut[0.5em] R_{q_{n}}(p) - R_{q_*}(p)}_{\text{estimation gap}}
    + \underbrace{\mystrut[0.5em] R_{q_*}(p) - R_p(p) }_{\text{approximation gap}}.
\end{equation}
These are the subjective risk analogues of approximation error, estimation error, and optimisation error \citep{bach2024learning}.
Every term of \autoref{eq:obj_decomp} scores two different {\em reports} on outcomes from reality, whereas every term of \autoref{eq:subj_decomp} scores the {\em same} report $p$ on outcomes from two different beliefs.
We can further decompose each term in \autoref{eq:subj_decomp} to expose a {\em grounding} term and an {\em uncertainty gap} term.

\begin{theorem}[Grounding and uncertainty by resource]
\label{thm:subjgaps}
   Under the assumptions of \autoref{thm:subjgap}, applied to each of $q$, $q_n$ and $q_*$ of Section~\ref{subsec:resources}, each gap in \autoref{eq:subj_decomp} reflects the effect of a single resource. Each gap can be written as a change in grounding, plus a change in predictive uncertainty: 
   \begin{alignat}{3}
&\textnormal{Optimisation gap:}\quad & R_q(p)-R_{q_n}(p) ~~&=~~ \underbrace{\mystrut[0.5em]G(q)-G(q_n)}_{\textnormal{\shortstack{optimisation\\grounding}}}
&&~~+~~ \underbrace{\mystrut[0.5em]U(q)-U(q_n)}_{\textnormal{\shortstack{optimisation\\uncertainty gap}}}, \label{eq:gap_opt}\\[6pt]
&\textnormal{Estimation gap:}\quad & R_{q_n}(p)-R_{q_*}(p) ~~&=~~ \underbrace{\mystrut[0.5em]G(q_n)-G(q_*)}_{\textnormal{\shortstack{estimation\\grounding}}}
&&~~+~~ \underbrace{\mystrut[0.5em]U(q_n)-U(q_*)}_{\textnormal{\shortstack{estimation\\uncertainty gap}}}, \label{eq:gap_est}\\[6pt]
&\textnormal{Approximation gap:}\quad & R_{q_*}(p)-R_p(p) ~~&=~~ \underbrace{\mystrut[0.5em]G(q_*)}_{\textnormal{\shortstack{approximation\\grounding}}}
&&~~+~~ \underbrace{\mystrut[0.5em]U(q_*)-U(p)}_{\textnormal{\shortstack{approximation\\uncertainty gap}}}. \label{eq:gap_app}
\end{alignat}
The three grounding terms add up to $G(q)$, and the three uncertainty gaps to $U(q)-U(p)$. Thus, \autoref{thm:subjgap} is recovered, with each of its two terms split by resource.
\end{theorem}

It is tempting to read each uncertainty gap in \autoref{thm:subjgaps} as the uncertainty `reducible' by one resource---by computation, data, or capacity. This is the language of the literature, where the estimation term would be called uncertainty `reducible by data' \citep{kendall2017,hullermeier2021concepts}. \autoref{thm:subjgaps} shows that this is imprecise, as supplying a resource can make the model {\em more uncertain}. For example, a model fitted to finite data may be {\em over-confident} \citep{guo2017calibration}, in which case more data would {\em increase} its uncertainty. For this reason, we describe the {\em change} in uncertainty that a resource brings, rather than the uncertainty it {\em reduces}.

More importantly, uncertainty alone cannot tell us whether the model has improved. When an over-confident model is given more data, its predictive uncertainty may increase while the divergence to reality decreases.  A change in uncertainty must therefore be read alongside the change in grounding---one tells us how confident the model has become, the other how much closer it has come to reality.

\begin{corollary}[The estimation uncertainty gap is the EUR]
\label{cor:eur}
Let $q_n$ be the model a learner returns on the training data $d_n$, and let $q_{n+m}$ be the model the same learner returns when $d_n$ is extended by $m$ further observations drawn i.i.d.\ from $p(\vx,y)$. Suppose the learner is consistent, in the sense that $q_{n+m}$ converges in probability to the best-in-class model $q_*$ as $m\to\infty$. Then the estimation uncertainty gap $U(q_n)-U(q_*)$ is the average over $X\sim p(\vx)$ of the expected uncertainty reduction of \citet[Section~5.2]{smith2024rethinking} under unlimited additional data.
\end{corollary}
\begin{corollary}[A non-zero approximation uncertainty gap implies misspecification]
\label{cor:misspec}
The approximation uncertainty gap $U(q_*)-U(p)$ captures the uncertainty due to model misspecification. It is zero whenever the class is well-specified $p\in\gH$, and therefore a non-zero gap indicates misspecification. This gap is positive when the best-in-class model is more uncertain than reality, and negative when it is more confident.
\end{corollary}
The notion of uncertainty due to model misspecification has been discussed at length by several authors, under a variety of definitions and notations \citep{hullermeier2021concepts,kato2023misspecification}. Corollary~\ref{cor:misspec} suggests a natural measure of it---the difference between the uncertainty of the best model the class can offer and the uncertainty of reality. Note that a zero approximation uncertainty gap does not imply a well-specified model. The two terms of the approximation gap therefore play different roles---the grounding tells us {\em whether} the class is misspecified, and the uncertainty gap tells us {\em how} this affects the model's uncertainty.
\subsection{Estimating Grounding and Uncertainty under Changes of Resource}
\label{subsec:exp_misspec}

We present two illustrative experiments for \autoref{thm:subjgaps} and Corollary~\ref{cor:misspec}, on a binary classification problem where reality $p$ is known, so that every term can be computed.
The true decision boundary is curved. We compare a {\em well-specified} neural network that can represent $p$ with a {\em misspecified} logistic regression whose boundary is linear.
Both are trained to convergence under log loss on $n$ i.i.d.\ pairs, so the optimisation gap is negligible and $q=q_n$.
In the first experiment we vary the amount of data, and in the second the capacity of the model.
The set-up and all details are given in Appendix~\ref{app:exp_misspec}.

\begin{figure}[h]
\centering
\includegraphics[width=\linewidth]{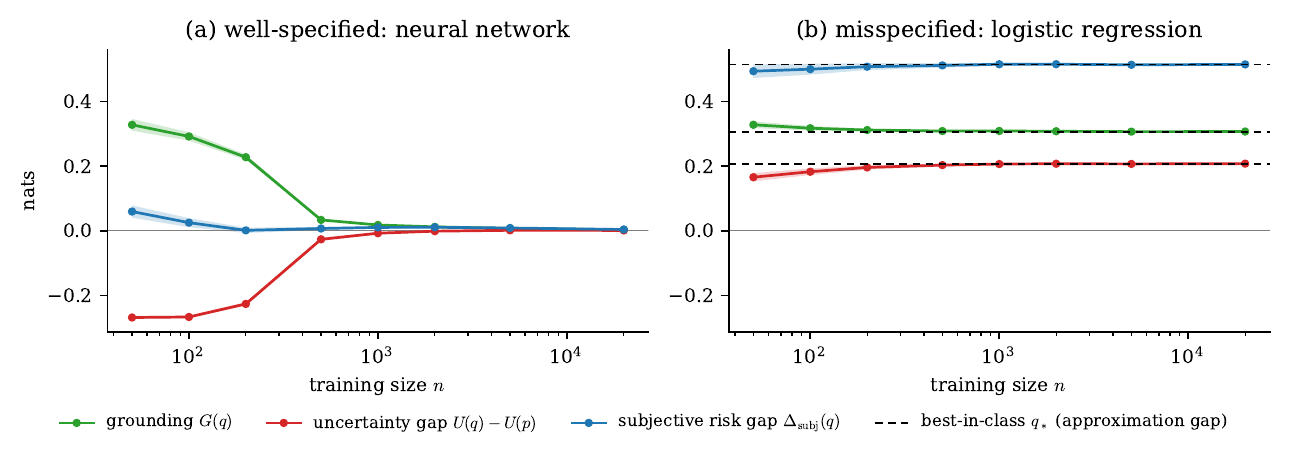}
\caption{Grounding, uncertainty gap and subjective risk gap as the training size $n$ grows, for a well-specified class (a) and a misspecified class (b). Each solid line is the mean over 100 repeated runs, with 95\% confidence intervals.}
\label{fig:exp_misspec}
\end{figure}

\autoref{fig:exp_misspec} shows the grounding, the uncertainty gap and their sum, i.e.\ the subjective risk gap (\autoref{thm:subjgap}), as $n$ grows. By \autoref{thm:subjgaps}, each of these splits into a part due to finite data (the estimation term) and a part due to the model class (the approximation term). With unlimited data the estimation term vanishes and only the approximation term remains. In \autoref{fig:exp_misspec}(a), the approximation term is negligible because the neural network can represent $p$, so all curves approach zero. In panel~(b) the dashed lines represent the approximation terms, and the vertical distance between each curve and its dashed line is the estimation term.

\paragraph{More data can increase uncertainty.}
In \autoref{fig:exp_misspec}(a), the model trained on few samples is over-confident and far from reality, as indicated by the negative uncertainty gap and the large grounding.
As $n$ grows, the grounding decreases while the predictive uncertainty {\em increases} until it matches the uncertainty of the true distribution.
More data can therefore make the model {\em more} uncertain while bringing it closer to reality, as discussed in Section~\ref{subsec:subjgaps}.

\paragraph{Irreducible depends on the resource.}
In \autoref{fig:exp_misspec}(b), the terms converge to those of $q_*$ rather than to zero. More data removes only the estimation terms, while the approximation gap remains.
As Corollary~\ref{cor:misspec} states, a non-zero approximation uncertainty gap indicates misspecification. Here the gap is positive, so even the best linear model is {\em more} uncertain than reality. This  uncertainty gap $U(q_*)-U(p)$ comes from the model class and cannot be removed by data, but a richer class can remove it, as we show next.

In the first experiment, we measured the impact of adding data for a well-specified and a misspecified class, and saw that data alone cannot remove the uncertainty that comes from the model class.
In the second experiment, we show the impact of changing the capacity of the model.
We keep the same setting and vary the width of the neural network, with each network trained to convergence. We repeat this with a few training samples ($n=200$) and with more data ($n=20{,}000$).
By \autoref{thm:subjgaps}, each curve in \autoref{fig:exp_capacity} is again the sum of an estimation term and an approximation term. With more data the estimation term is small, so panel~(b) shows mainly the approximation term. With few samples, panel~(a) shows both.
\begin{figure}[t]
\centering
\includegraphics[width=0.89\linewidth]{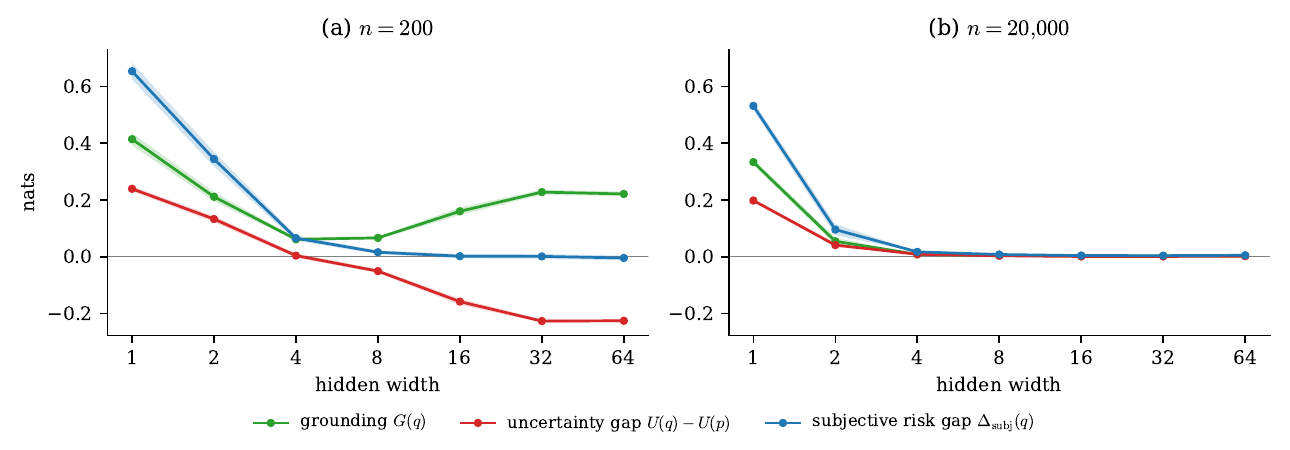}
\caption{Grounding, uncertainty gap and subjective risk gap as the width of the neural network grows, with few samples (a) and many samples (b). Each solid line is the mean over 100 repeated runs, with 95\% confidence intervals.}
\label{fig:exp_capacity}
\end{figure}
\paragraph{Capacity removes the approximation gap.}
In \autoref{fig:exp_capacity}(b), the networks with the fewest hidden units (1 and 2) are misspecified. They are far from reality and {\em more} uncertain than reality.
Under the aleatoric/epistemic view, this extra uncertainty would be read as noise in the data, since more data does not remove it. Yet as the width grows, both the grounding and the uncertainty gap vanish.
What looked like `aleatoric’ uncertainty is therefore removed by capacity.

\paragraph{Capacity must be balanced with the available data.}
In \autoref{fig:exp_capacity}(a), the uncertainty of the model decreases as the width grows---wider networks are more confident.
This agrees with the observation of \citet{guo2017calibration} that larger networks tend to be over-confident.
With few training samples, wide networks can almost memorise the data, so the estimation term dominates---the model becomes over-confident and its grounding grows again.
The grounding is therefore smallest at an intermediate width, where the approximation and estimation terms balance, and not at the width with the lowest uncertainty.
This is the approximation--estimation trade-off of statistical learning theory, seen through uncertainty. Small capacity makes the model under-confident, and large capacity for the available data makes it over-confident.
The uncertainty gap tells us whether the model is under- or over-confident, and only the grounding tells us how close it is to reality.

\subsection{The GDE Decomposition under Changes of Resource}
\label{subsec:gde_resources}

We now consider the scenario of \autoref{thm:gde_decomposition}, where we have uncertainty over what our deployed predictor will be. In this case the appropriate evaluation is an {\em expected} risk. Integrating \autoref{thm:gde_decomposition} over $p(\vx)$, the expected subjective risk has the following decomposition,
\begin{eqnarray}
    \E_\Theta\Big[  R_{q_\Theta}(p) \Big]
    &=&
    \E_{X\sim p(\vx)}\Big[ \E_\Theta\big[~  B_\phi(q_\Theta, p) + H_\ell(q_\Theta) ~\big] ~\Big\vert~ X \Big]  \notag\\
    &=&
    \E_{X\sim p(\vx)}\Big[ \textcolor{blue}{B_\phi(\bar{q}, p)} + \textcolor{BrightOlive}{\E_\Theta\big[ B_\phi(q_\Theta, \bar{q}) \big]} + \textcolor{red}{\E_\Theta\big[H_\ell(q_\Theta)\big]} ~\Big\vert~ X \Big].
    \label{eq:gde_integrated}
\end{eqnarray}
Here $\textcolor{blue}{B_\phi(\bar q,p)}$ is the {\em grounding} term, connecting the expected model $\bar q$ to reality $p$. The term $\textcolor{BrightOlive}{\E_\Theta\big[ B_\phi(q_\Theta, \bar{q}) \big]}$ is the predictive disagreement, and $\textcolor{red}{\E_\Theta\big[H_\ell(q_\Theta)\big]}$ is the average predictive entropy. Thus, the terms capture downstream consequences of {\em parameter uncertainty}, i.e.\ $\Theta$. In the case of log-loss, the latter two are Gal's information-theoretic uncertainty terms.

\paragraph{The GDE decomposition and the subjective risk gap.}
Consider the random model $q_\Theta$ from \autoref{thm:gde_decomposition} and its mean prediction $\bar q=\E_\Theta[q_\Theta]$, at a fixed input $\vx$.\footnote{The same holds in expectation over $X\sim p(\vx)$.} Because subjective risk is linear in the belief, $\E_\Theta\big[\Delta_{\mathrm{subj}}(q_\Theta)\big]=\Delta_{\mathrm{subj}}(\bar q)$. The gap is the same in expectation whether we deploy a randomly selected member or the mixture. For a randomly deployed member, applying \autoref{thm:subjgap} to each member gives
\begin{equation}
\E_\Theta\big[\Delta_{\mathrm{subj}}(q_\Theta)\big] ~=~ %
\underbrace{\mystrut[0.6em]\textcolor{blue}{B_\phi(\bar q,p)} + \textcolor{BrightOlive}{\E_\Theta\big[B_\phi(q_\Theta,\bar q)\big]}}_{\textnormal{expected grounding of a member}} ~+~ \underbrace{\mystrut[0.6em]\textcolor{red}{\E_\Theta\big[H_\ell(q_\Theta)\big]} - H_\ell(p)}_{\textnormal{expected uncertainty gap}}.
\label{eq:scenario_member}
\end{equation}
For the mixture it gives
\begin{equation}
\Delta_{\mathrm{subj}}(\bar q) ~=~
\underbrace{\mystrut[0.6em]\textcolor{blue}{B_\phi(\bar q,p)}}_{\textnormal{grounding of the mixture}} +~\underbrace{\mystrut[0.6em]\textcolor{BrightOlive}{\E_\Theta\big[B_\phi(q_\Theta,\bar q)\big]} + \textcolor{red}{\E_\Theta\big[H_\ell(q_\Theta)\big]} - H_\ell(p)}_{\textnormal{uncertainty gap of the mixture}}. %
\label{eq:scenario_mixture}
\end{equation}
Both scenarios contain the same terms. Under log-loss, the disagreement is the mutual information $I(\hat Y;\Theta)$, interpreted by \citet{gal2016uncertainty} as `epistemic' uncertainty. The subjective risk gap gives this term two readings. For a single member, it contributes to expected grounding---the cost of not knowing which parameter is deployed. For the mixture, it contributes to the uncertainty gap---the cost of averaging over all plausible parameters. Whether it counts as grounding or as uncertainty therefore depends on what we deploy, and whether more data would reduce it is a separate question, to which we now turn.

\paragraph{Adding data.}
However, this formulation is {\em static}: it makes no reference to training, more data, or any other resource. To incorporate this, we condition on the data. With $d_n$ fixed, $\Theta=T(d_n,A)$ is random only through $A$; we write $q_n:=q_\Theta$ for the (random) model trained on $d_n$, and $\bar q_n:=\E_{\Theta\mid d_n}[q_n]$. Here $q_n=q_\Theta$ remains an empirical risk minimiser almost
surely, as in Equation~\ref{eq:subj_decomp}; the training randomness
may select among empirical risk minimisers. The decomposition is then,
\begin{equation}
    \E_{\Theta\mid d_n}\Big[  R_{q_n}(p) \Big]
    =
    \E_{X\sim p(\vx)}\Big[ \textcolor{blue}{B_\phi(\bar{q}_n, p)} + \textcolor{BrightOlive}{\E_{\Theta\mid d_n}\big[ B_\phi(q_n, \bar{q}_n) \big]} + \textcolor{red}{\E_{\Theta\mid d_n}\big[H_\ell(q_n)\big]} ~\Big\vert~ X \Big].
    \label{eq:gde_dn}
\end{equation}
As data are added, $\Theta\mid d_{n+m}$, the extreme case is $q_\infty$. If $\Theta\mid d_\infty$ is a Dirac distribution on the best-in-class model, then $q_\infty=q_*$, the disagreement vanishes, and \autoref{eq:gde_dn} becomes $R_{q_*}(p)=\E_{X\sim p(\vx)}\big[B_\phi(q_*,p)+H_\ell(q_*)\mid X\big]$. Taking the difference gives the following.

\begin{corollary}[Expected estimation gap and the predictive mixture]
\label{cor:gde_estimation}
Suppose $\Theta\mid d_\infty$ is a Dirac distribution on a parameter whose model is the best-in-class $q_*$. Then
\begin{eqnarray}
    \E_{\Theta\mid d_n}\Big[  R_{q_n}(p) \Big]
    -
    R_{q_*}(p)
    &=&
    \E_{X\sim p(\vx)}\Big[ \textcolor{blue}{B_\phi(\bar{q}_{n}, p)}-B_\phi(q_*, p) \notag\\
    &+&
    \textcolor{BrightOlive}{\E_{\Theta\mid d_n}\big[ B_\phi(q_n, \bar{q}_n) \big]} + \textcolor{red}{\E_{\Theta\mid d_n}\big[H_\ell(q_n)\big]}  - H_\ell(q_*)
    ~\Big\vert~X\Big] \label{eq:cor.1}\\
    &=&
    \underbrace{\mystrut[0.5em]G(\bar q_n)-G(q_*)}_{\textnormal{estimation grounding}}
    ~+~
    \underbrace{\mystrut[0.5em]U(\bar q_n)-U(q_*)}_{\textnormal{estimation uncertainty gap}}.
    \label{eq:cor.2}
\end{eqnarray}
That is, the expected estimation gap equals the subjective risk
difference $R_{\bar q_n}(p)-R_{q_*}(p)$, and its uncertainty term
is the expected uncertainty reduction of \citet{smith2024rethinking}
for the predictive $\bar q_n$, under unlimited additional data
drawn from $p$.
\end{corollary}
If training on $d_n$ always returns the same model (i.e.\ there is no parameter uncertainty), the mixture $\bar q_n$ is just that model, and Corollary~\ref{cor:gde_estimation} reduces to Corollary~\ref{cor:eur}.

\subsection{Summary}

The subjective risk gap splits into a grounding term and an uncertainty gap
(\autoref{thm:subjgap}). Each term then splits by resource
(\autoref{thm:subjgaps}), mirroring the decomposition of excess risk in SLT.
This lets us ask how additional computation, data, or model capacity changes
both the model's uncertainty and how close its predictions are to reality.
Uncertainty is reducible only {\em with respect to a stated
resource}, and a change in uncertainty must be read alongside a change
in grounding.  For example, more data may make an over-confident model
{\em more} uncertain while bringing it closer to reality. Conversely, a model
can become more confident while moving further from reality. Stating the
resource makes clear what is allowed to change and gives us a precise basis
for using the terms `aleatoric' and `epistemic'.

\clearpage
\newpage

\section{Discussion of Related Work}
\label{sec:discussion}

We present a brief discussion of relations to previous work.
\subsection{The Debate Over the Aleatoric/Epistemic Dichotomy}
The closest work to our own is \citet{smith2024rethinking}, who
identify incoherence in existing discussions 
across the UQ literature. They argue that the aleatoric/epistemic view has capacity for only two concepts, whilst the community needs several: predictive uncertainty, its reduction with more data, predictive performance, and data dispersion.
Following \citet{degroot1962}, they define predictive uncertainty as the subjective risk of the Bayes-optimal action, %
and `reducible' uncertainty as its expected reduction once future data is observed under a data-acquisition policy and an updating scheme.
For predictive performance and data dispersion, they separately examine the (objective) risk and entropy of a reference distribution $p_{eval}$.

\citet{sale2026meaningful} defend the dichotomy, arguing that the split between reducible and irreducible uncertainty is relative to a stated modelling setup. What counts as reducible depends on what that setup allows to change---more data, additional features, a richer hypothesis class---so the same uncertainty may be aleatoric in one setting and epistemic in another. On this view, recent criticism, including that of \citet{smith2024rethinking}, targets particular formalisations or implicit assumptions rather than the dichotomy itself.

We agree that the terms are meaningful once the learning scenario is stated precisely, but they are overloaded in the literature. Since uncertainty is reducible or irreducible only given {\em a specific resource}, we must be explicit about the interventions. Moreover, uncertainty alone can be misleading, since a model may be entirely certain about a wrong prediction. To formalise our position, we share the starting point of \citet{smith2024rethinking}, but evaluate the subjective risk of the true distribution, so that external grounding appears alongside the uncertainty terms (Section~\ref{sec:GDE}), and reducibility is tied to specific resources (Section~\ref{sec:SLT}). Our hope is that this allows the terms `aleatoric' and `epistemic' to be used with precision, rather than as loose labels.

\subsection{Risk-Based Views of Uncertainty}
Several proposals define uncertainty terms by decomposing a risk ~\citep{Kotelevskii2022,lahlou2023}. A risk functional scores a reported distribution against outcomes drawn from an outcome distribution: the objective risk $L_p(q)$ draws outcomes from reality and scores the model, while our subjective risk $L_q(p)$ draws them from the model and scores reality. Existing proposals differ in which distribution plays each role, and in which of the two is treated as random; each choice leads to a different decomposition. 

\paragraph{Incomplete measures.} \citet{jimenez_2026position} present a position arguing that epistemic uncertainty estimation methods are fundamentally incomplete.
Following \citet{shaker2020aleatoric}, they consider only the {\em objective} risk and its properties in statistical learning theory, drawing on results from \citet{brown2024notthesame} to argue that the community’s understanding of uncertainty is incomplete. They consider model bias as a part of epistemic uncertainty that current estimators fail to capture.

\paragraph{Evaluation against oracle uncertainty.} Since ground truth is rarely available, UQ methods are usually evaluated on proxy tasks. \citet{wizgall2026groundtruthframeworkuncertaintydisentanglement} address this by constructing a semi-synthetic regression benchmark where the oracle EU/AU are known.

Both \citet{jimenez_2026position} and \citet{wizgall2026groundtruthframeworkuncertaintydisentanglement}
conduct their experiments under MSE. Under squared loss, the objective
and subjective regret terms coincide: both equal
$(\mu_p-\mu_q)^2$. Their full risks can nevertheless differ, because
their entropy terms are respectively $\operatorname{Var}_p(Y)$ and
$\operatorname{Var}_q(\hat Y)$. Asymmetric losses introduce a further
distinction through the orientation of the divergence
(Section~\ref{sec:goodUncertaintyMeasure}).

\paragraph{A family of measures from proper scoring rules.} Just as in our work, \citet{proper_regression} and~\citet{kotelevskii2025from} also derive uncertainty measures from proper scoring rules and risk decomposition.
Their starting point however is the `objective' risk.
They replace the unavailable true distribution with Bayesian estimates,
which can be inserted into either argument of the induced divergence. By enumerating all combinations of argument orderings, they recover many proposed epistemic and aleatoric terms as special cases.

The ``all combinations'' strategy produces the measures, but does not say which is appropriate for a given modelling scenario, or explain why they behave differently. We instead fix the modelling and evaluation choices first and obtain a single grounding--disagreement--entropy decomposition in which each term has a clear role. The particular measure that appears is then a {\em consequence} of the loss geometry rather than a free choice, and disagreements between measures are explained directly by differences in loss curvature (Section \ref{sec:goodUncertaintyMeasure}). We now examine the measures that arise from varying the arguments and their order.

\subsection{The Various Proposals for `Epistemic' Uncertainty}\label{sec:relations}

The seminal work of \citet{gal2016uncertainty} introduced the world to {mutual information} as a measure of epistemic uncertainty. In particular, written in the form $I(\hat{Y};\Theta) = \E_\Theta\left[D_{\mathrm{KL}}(q_\Theta\mid\mid\bar{q})\right]$, it has an appealing
structure that has led to multiple very similar proposals.
For simplicity, we restrict discussion to the case of KL divergences on categorical distributions, but the conclusions hold more generally for arbitrary Bregman divergences. 
\autoref{tab:manyforms} lists the proposals, in terms of the right and left centroids, $\bar q$ and $\mathring{q}$.
Note that, for categorical predictions, these are functionals of the second-order distribution of
$q_\Theta(\cdot\mid\vx)$ described in Section~\ref{sec:GDE}.

\begin{table}[b]
\centering\renewcommand{\arraystretch}{1.5}%
\begin{tabular}{lll}
\toprule
\bf Name & \bf KL form & \bf First suggested by \\
\midrule
Mutual Information & $\E_{\Theta} [ D_{\mathrm{KL}} (q_\Theta \mid\mid \bar{q}) ]$ & \small \citet{gal2016uncertainty} \\
Expected Pairwise KL (EPKL) & $\E_{\Theta} \E_{\Theta'} \left[ D_{\mathrm{KL}}(q_\Theta \mid\mid q_{\Theta'}) \right]$ & \small\citet{malinin_2019}\\
Reverse Mutual Information (RMI) & $\E_\Theta [ D_{\mathrm{KL}}(\bar{q}\mid\mid q_\Theta) ]$ & \small \citet{malinin2021uncertainty} \\
Modified Bregman information & $\E_\Theta [ D_{\mathrm{KL}}(q_\Theta \mid\mid \mathring{q}) ]$ & \small\citet{kotelevskii2025from}\\
Reverse modified Bregman information & $\E_\Theta[ D_{\mathrm{KL}}(\mathring{q}\mid\mid q_\Theta) ]$ & \small\citet{kotelevskii2025from}\\

\bottomrule
\end{tabular}
\caption{KL-based proposals for `epistemic' uncertainty.} 
\label{tab:manyforms}
\end{table}

These measures follow a common template of the form $\E_\Theta\big[D_{\mathrm{KL}}(a,b)\big]$, enumerating all possible substitutions of arguments $a,b$---e.g. with $a=\mathring{q}$ and $b=q_\Theta$, we get the `reverse modified Bregman information'.
We note that only {\em two} of the terms are
`standard' variances in the sense that they are a Jensen gap.
These are the first and last rows above: $\E_\Theta[D_{\mathrm{KL}}(q_\Theta \mid\mid \bar{q})]$, and $\E_\Theta[D_{\mathrm{KL}}(\mathring{q}\mid\mid q_\Theta)]$. Reversing the order of arguments, as in RMI and Modified Bregman information, yields non-standard quantities. Proposition~\ref{prop:centroid-decomp} shows how they relate to the standard ones.
\begin{proposition}[Centroid decomposition of the non-standard measures]
\label{prop:centroid-decomp}
Let $q_\Theta$, with $\Theta$ as in Theorem~\ref{thm:gde_decomposition}, be categorical distributions over $\mathcal Y$,
with arithmetic mean $\bar q=\E_\Theta[q_\Theta]$ and normalised geometric mean
$\mathring q\propto\exp\E_\Theta[\ln q_\Theta]$.
Assume there exists at least one class $y \in \gY$ where $q_\Theta(y) > 0$ almost surely, and $\E_\Theta[\ln q_\Theta(y)] > -\infty$. Then,
\begin{align}
\E_\Theta \big[D_{\mathrm{KL}}(q_\Theta\mid\mid \mathring q)\big]
   &= \underbrace{\E_\Theta\big[D_{\mathrm{KL}}(q_\Theta\mid\mid \bar q)\big]}_{\text{mutual information}}
   + \underbrace{D_{\mathrm{KL}}(\bar q\mid\mid \mathring q)}_{\text{systematic gap}}, \label{eq:mod-decomp}\\
   \E_\Theta \big[D_{\mathrm{KL}}(\bar q\mid\mid q_\Theta)\big]
   &= \underbrace{\E_\Theta\big[D_{\mathrm{KL}}(\mathring q\mid\mid q_\Theta)\big]}_{\text{variance}}
   + \underbrace{D_{\mathrm{KL}}(\bar q\mid\mid \mathring q)}_{\text{systematic gap}}. \label{eq:rmi-decomp}
\end{align}
Both share the systematic gap $D_{\mathrm{KL}}(\bar q\mid\mid \mathring q)\ge 0$, which vanishes iff $\bar q=\mathring q$.
\end{proposition}

These non-standard quantities can therefore report `more uncertainty' not because the models are more spread out, but because $\bar q$ and $\mathring q$ diverge structurally---a systematic effect, rather than genuine disagreement among the models.

The gap also matters for finiteness. Under the assumption of Proposition~\ref{prop:centroid-decomp} both standard variances are finite, so Equations~\ref{eq:mod-decomp} and \ref{eq:rmi-decomp} imply that each non-standard measure is infinite if and only if the systematic gap is. This happens whenever some but not all models assign a class zero probability. For such a class, $\mathring q(y) = 0 < \bar q(y)$, and hence $D_{\mathrm{KL}}(\bar q \mid\mid \mathring q) = \infty$. For example, with $q_1 = (\tfrac12, \tfrac12, 0)$ and $q_2 = (\tfrac13, \tfrac13, \tfrac13)$: $\mathring q$ gives the third class zero mass while $\bar q$ gives it $\tfrac16$, so the gap and both non-standard measures are infinite.

The Expected Pairwise KL \citep{malinin_2019}, analysed in depth by \citet{schweighofer2023introducing}, is intuitively appealing as it averages all pairwise disagreements. However, like the non-standard measures, it can be infinite as it carries the systematic gap.

\subsection{The Axiomatic Approach, e.g. \citet{pmlr-v216-wimmer23a}}
\citet{pmlr-v216-wimmer23a} showed that the entropy-based measures violate several axioms, three of which we consider here: that more spread in the second-order distribution should strictly increase uncertainty (A3), that the uniform distribution represents maximal uncertainty, i.e.\ full ignorance (A2), and that a location shift preserving spread should leave uncertainty unchanged (A5).\footnote{Axiom numbering varies across papers; we follow
\citet{pmlr-v216-wimmer23a}.} However, once we view uncertainty as a derived component, these axioms can be understood as constraints on modelling choices, rather than as universal requirements on the measure itself.

\paragraph{Total uncertainty depends only on the mixture.}
Consider axiom~A3. The measures that \citet{pmlr-v216-wimmer23a} call aleatoric and epistemic are, in our terms, the average predictive entropy and the predictive disagreement. By Equation~\ref{eq:gde_as_ltv}, these sum to $H_\ell(\bar q)$, a functional of $\bar q$ alone, so any two distributions over $\Theta$ that induce the same predictive mixture share the same total uncertainty. For example, under log-loss, the uniform distribution over $[0,1]$ and the Dirac mixture $\tfrac{1}{2}\delta_0 + \tfrac{1}{2}\delta_1$ over the Bernoulli parameter $\Theta$ both induce $\bar q = \mathrm{Ber}(\tfrac{1}{2})$, and hence the same total uncertainty. Differences in the shape of the distribution over $\Theta$ can still appear in how this total splits between disagreement and entropy.
\paragraph{Should the uniform distribution carry more epistemic uncertainty than the Dirac mixture?} This is axiom A2. What makes the disagreement large is not how widely $\Theta$ is distributed, but how far the resulting models sit from $\bar{q}$. The uniform spreads mass across the interior of the simplex, where average divergence to $\bar{q}$ is small ($\ln 2 - \tfrac{1}{2} \approx 0.19$ nats); the Dirac mixture places every model at a vertex, as far from $\bar q$ as possible ($\ln 2 \approx 0.69$ nats). 

\paragraph{Should an uncertainty measure be translation invariant?}

Axiom~A5 requires that translating a second-order
distribution while preserving its spread should leave the epistemic uncertainty
unchanged. The variance-based measures satisfy this
\citep{salelabel}, while the entropy-based measures do not
\citep{pmlr-v216-wimmer23a}. Indeed, shifting a uniform distribution on $[0.45, 0.85]$ to $[0.3, 0.7]$ preserves its width, yet the mutual information falls from $0.030$ to $0.027$ nats. This follows from Bregman geometry---the disagreement is $\E_\Theta[B_\phi(q_\Theta, \bar q)]$, and a Bregman divergence measures distance through the curvature of its generator $\phi$ (Section~\ref{sec:goodUncertaintyMeasure}). Under squared loss, $\phi$ has constant curvature: $B_\phi(a,b) = (a-b)^2$ depends only on the difference between its arguments, so a translation leaves the disagreement unchanged. Under log-loss, the curvature of the negative entropy grows towards the boundary of the simplex, so the same spread weighs more near the vertices. Figure~\ref{fig:translation-simplex} shows how translation acts in different coordinate systems. 

\begin{figure}[!htb]
    \centering
    \includegraphics[width=0.6\linewidth]{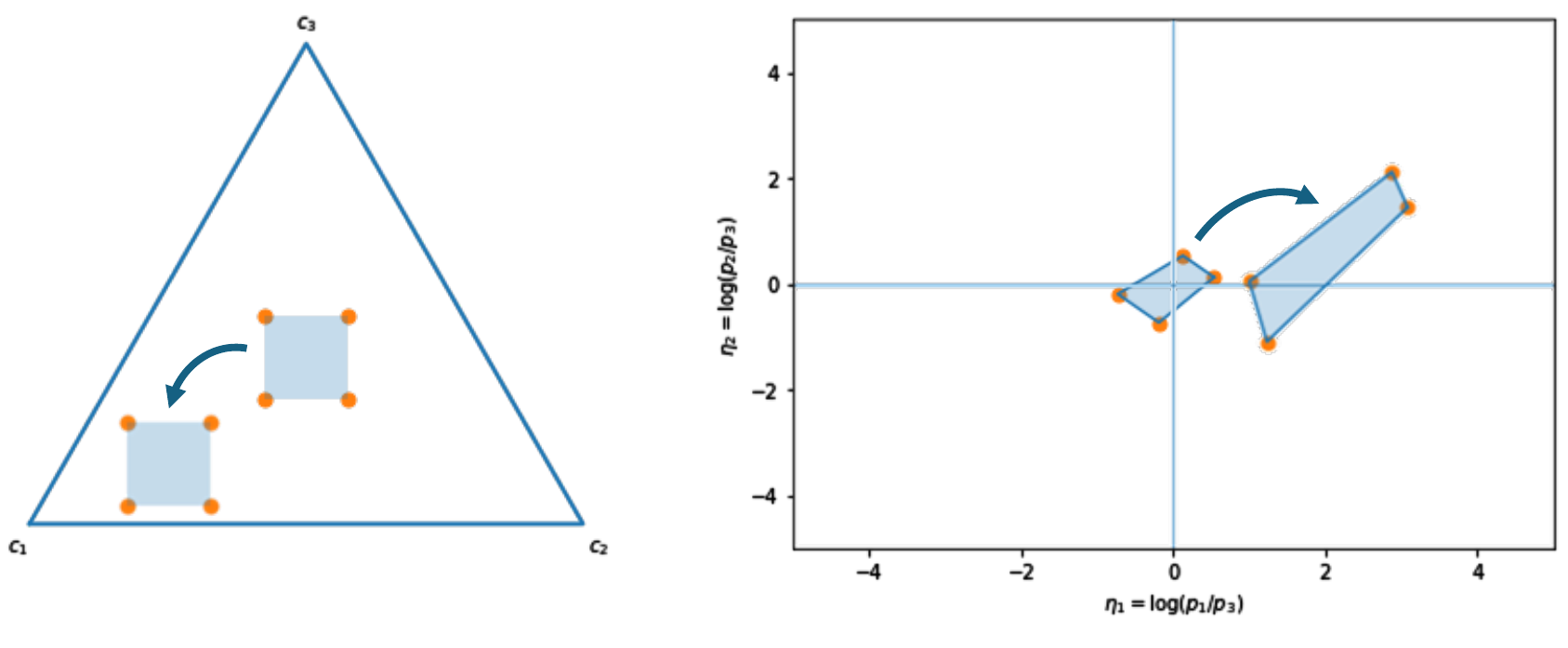}
    \caption{{\bf Translating a credal set in the simplex.} 
Left: a credal set is moved towards a vertex, with a fixed volume ($0.0225$). Right: the same sets in natural-parameter (log-odds) coordinates, where volume and shape change ($0.76$ to $2.2$). The stretching near the vertex reflects the growing curvature of the negative entropy.}
    \label{fig:translation-simplex}
    \vspace{-0.5\baselineskip}
\end{figure}

From this view, the axioms of \citet{pmlr-v216-wimmer23a} do not disagree with information-theoretic UQ itself, but with the modelling choices from which it is derived.

\subsection{Summary}

We have considered recent literature with the view that uncertainty is a derived quantity of the subjective risk. %
We suggest that many disagreements in the literature reflect differences in the loss geometry and modelling choices, rather than problems or contradictions.

\clearpage\newpage

\section{Conclusion}

We have presented a new view of uncertainty quantification based on the decision-theoretic notion of {\em subjective risk}.
The value of this proposition is that it replaces a collection of seemingly distinct measures with a single underlying principle, while providing a natural bridge between uncertainty quantification and statistical learning theory.

Our approach rests on {\em decomposing} subjective risk expressions into components with meaningful interpretations for UQ.
Thus, rather than defining measures of `aleatoric' and `epistemic' uncertainty, and debating which is correct, %
we start from the model and the loss by which its predictions are judged.
The decomposition of the corresponding expected {\em subjective risk} then determines the uncertainty measures appropriate to that learning scenario, rather than requiring them to be specified separately.
For example, under log-loss with categorical outcomes, subjective risk takes the form of reverse cross-entropy, and its decomposition recovers the classic information-theoretic quantities of \citet{gal2016uncertainty}.
With other distributional/loss assumptions, the same principle recovers proposals that have appeared independently in the past decade of UQ literature \citep{kendall2017,depeweg,huang2021quantifying, proper_regression,bulte2026,jimenez_2026position}.
The uncertainty measures can be characterised in a unified manner by the Bregman geometry induced by the choice of proper loss.

Subjective risk provides not only a way to unify existing notions of uncertainty, but also a route to take the field in new directions.
Since risk is the fundamental object of {\em statistical learning theory}, our subjective risk framework allows uncertainty to be connected directly to learning dynamics: how uncertainty changes with the data available, model mis-specification, and the ability of the learning algorithm to optimise within its function class.
To understand this, we defined and analysed subjective-risk analogues of {\em approximation}, {\em estimation}, and {\em optimisation} error \citep{bach2024learning}. %
Each term can be further decomposed into a change in {\em grounding} and a change in predictive {\em uncertainty}.
We showed how the recent proposal of \citet{smith2024rethinking} corresponds to the uncertainty gap of the estimation term, but does not address the broader context of how uncertainty can change with other learning resources, such as expanding the function class.
In this light, uncertainty becomes a `first-class citizen' of the learning scenario, rather than requiring post-hoc definitions.

In short, we propose subjective risk as a central object of study for uncertainty quantification.
Our hope is that this will help establish stronger theoretical foundations for the field: if {\em objective} risk is the organising principle underlying supervised learning and statistical learning theory, {\em subjective} risk might play the same role for uncertainty quantification.

\acks{Raghad Alamri gratefully acknowledges support from the Saudi Ministry of Education, through Taibah University.
Michele Caprio acknowledges support from the {Prob\_AI} Hub and the London Mathematical Society.
OpenAI's ChatGPT Astra was used to assist with proofreading, and with improving the clarity and presentation of the manuscript.}

\newpage

\bibliography{refs}

@inproceedings{kato2023misspecification,
  author    = {Kato, Yuko and Tax, David M. J. and Loog, Marco},
  title     = {A View on Model Misspecification in Uncertainty Quantification},
  booktitle = {Artificial Intelligence and Machine Learning (BNAIC/Benelearn 2022), Revised Selected Papers},
  series    = {Communications in Computer and Information Science},
  volume    = {1805},
  pages     = {65--77},
  publisher = {Springer},
  year      = {2023},
  doi       = {10.1007/978-3-031-39144-6_5}
}

@inproceedings{malinin-rkl-2019,
  title={Reverse KL-Divergence Training of Prior Networks: Improved Uncertainty and Adversarial Robustness},
  author={Malinin, Andrey and Gales, Mark},
  booktitle={Neural Information Processing Syst.},
  year={2019}
}

@inproceedings{
kirchhof2025reexamining,
title={Reexamining the Aleatoric and Epistemic Uncertainty Dichotomy},
author={Michael Kirchhof and Gjergji Kasneci and Enkelejda Kasneci},
booktitle={The Fourth Blogpost Track at ICLR 2025},
year={2025},
url={https://openreview.net/forum?id=lE7WZ2DpMq}
}

@article{savage1971elicitation,
  title={Elicitation of personal probabilities and expectations},
  author={Savage, Leonard J},
  journal={Journal of the American Statistical Association},
  volume={66},
  number={336},
  pages={783--801},
  year={1971},
  publisher={Taylor \& Francis}
}

@article{bottou2007tradeoffs,
  title={The tradeoffs of large scale learning},
  author={Bottou, L{\'e}on and Bousquet, Olivier},
  journal={Advances in neural information processing systems},
  volume={20},
  year={2007}
}

@misc{kirsch2024twitter,
  author       = {Various},
  title        = {X discussion thread on uncertainty and machine learning},
  year         = {2024},
  month        = {July},
  howpublished = {\url{https://x.com/BlackHC/status/1817556167687569605}},
  note         = {X (formerly Twitter) thread, accessed 2026-05-29}
}

@article{banerjee2005optimality,
  title={On the optimality of conditional expectation as a Bregman predictor},
  author={Banerjee, Arindam and Guo, Xin and Wang, Hui},
  journal={IEEE Trans. Information Theory},
  volume={51},
  number={7},
  pages={2664--2669},
  year={2005},
  publisher={IEEE}
}

@article{hullermeier2021concepts,
  title={Aleatoric and epistemic uncertainty in machine learning: An introduction to concepts and methods},
  author={H{\"u}llermeier, Eyke and Waegeman, Willem},
  journal={Machine learning Journal},
  year={2021},
  publisher={Springer},
}

@inproceedings{brown2024notthesame,
  title={Bias-Variance is not the same as Approximation-Estimation},
  author={Gavin Brown and Riccardo Ali},
  booktitle={Transactions on Machine Learning Research},
  year=2024
}

@inproceedings{kull2015novel,
  title={Novel decompositions of proper scoring rules for classification: Score adjustment as precursor to calibration},
  author={Kull, Meelis and Flach, Peter},
  booktitle={Machine Learning and Knowledge Discovery in Databases: European Conference, ECML PKDD 2015, Porto, Portugal, September 7-11, 2015, Proceedings, Part I 15},
  pages={68--85},
  year={2015},
  organization={Springer}
}

@article{schweighofer2023introducing,
  title={Introducing an improved information-theoretic measure of predictive uncertainty},
  author={Schweighofer, Kajetan and Aichberger, Lukas and Ielanskyi, Mykyta and Hochreiter, Sepp},
  journal={arXiv preprint arXiv:2311.08309},
  year={2023}
}

@inproceedings{shaker2020aleatoric,
  title={Aleatoric and epistemic uncertainty with random forests},
  author={Shaker, Mohammad Hossein and H{\"u}llermeier, Eyke},
  booktitle={Advances in Intelligent Data Analysis XVIII: 18th International Symposium on Intelligent Data Analysis, IDA 2020, Konstanz, Germany, April 27--29, 2020, Proceedings 18},
  pages={444--456},
  year={2020},
  organization={Springer}
}

@phdthesis{gal2016uncertainty,
  title={{Uncertainty in Deep Learning}},
  author={Gal, Yarin},
  year={2016},
  school={{University of Cambridge}}
}

@misc{sale2023uncertainty,
      title={Second-Order Uncertainty Quantification: Variance-Based Measures}, 
      author={Yusuf Sale and Paul Hofman and Lisa Wimmer and Eyke Hüllermeier and Thomas Nagler},
      year={2023},
      eprint={2401.00276},
      archivePrefix={arXiv},
      primaryClass={cs.LG},
      url={https://arxiv.org/abs/2401.00276}, 
}

@inproceedings{malinin2021uncertainty,
  title={Uncertainty Estimation in Autoregressive Structured Prediction},
  author={Malinin, Andrey and Gales, Mark},
  booktitle={International Conference on Learning Representations},
 year= {2021}
}

@inproceedings{depeweg,
  title={Decomposition of uncertainty in Bayesian deep learning for efficient and risk-sensitive learning},
  author={Depeweg, Stefan and Hernandez-Lobato, Jose-Miguel and Doshi-Velez, Finale and Udluft, Steffen},
  booktitle={International conference on machine learning},
  pages={1184--1193},
  year={2018},
  organization={PMLR}
}

@misc{lahlou2023,
      title={DEUP: Direct Epistemic Uncertainty Prediction}, 
      author={Salem Lahlou and Moksh Jain and Hadi Nekoei and Victor Ion Butoi and Paul Bertin and Jarrid Rector-Brooks and Maksym Korablyov and Yoshua Bengio},
      year={2023},
      eprint={2102.08501},
      archivePrefix={arXiv},
      primaryClass={cs.LG},
      url={https://arxiv.org/abs/2102.08501}, 
}

@misc{gruber2023properBiasVariance,
      title={Uncertainty Estimates of Predictions via a General Bias-Variance Decomposition}, 
      author={Sebastian G. Gruber and Florian Buettner},
      year={2023},
      eprint={2210.12256},
      archivePrefix={arXiv},
      primaryClass={cs.LG},
      url={https://arxiv.org/abs/2210.12256}, 
}

@phdthesis{malinin_2019, title={Uncertainty Estimation in Deep Learning with application to Spoken Language Assessment}, url={https://www.repository.cam.ac.uk/handle/1810/298857}, DOI={10.17863/CAM.45912}, school={University of Cambridge}, author={Malinin, Andrey}, year={2019} }

@article{sale2026meaningful,
  title={{The Aleatoric-Epistemic Dichotomy of Uncertainty is Meaningful and Indispensable for Machine Learning}},
  author={Sale, Yusuf and Kotelevskii, Nikita and Panov, Maxim and Hüllermeier, Eyke},
  date={May 17th 2026},
  journal={SSRN},
  url={https://ssrn.com/abstract=6783338},
  year={2026}
}

@misc{ziegler2020finetuninglanguagemodelshuman,
      title={Fine-Tuning Language Models from Human Preferences}, 
      author={Daniel M. Ziegler and Nisan Stiennon and Jeffrey Wu and Tom B. Brown and Alec Radford and Dario Amodei and Paul Christiano and Geoffrey Irving},
      year={2020},
      eprint={1909.08593},
      archivePrefix={arXiv},
      primaryClass={cs.CL},
      url={https://arxiv.org/abs/1909.08593}, 
}

@article{xiong2023iterative,
  title={Iterative preference learning from human feedback: Bridging theory and practice for rlhf under kl-constraint},
  author={Xiong, Wei and Dong, Hanze and Ye, Chenlu and Wang, Ziqi and Zhong, Han and Ji, Heng and Jiang, Nan and Zhang, Tong},
  journal={arXiv preprint arXiv:2312.11456},
  year={2023}
}

@inproceedings{cao2025on,
title={On {LLM} Knowledge Distillation - A Comparison between Forward {KL} and Reverse {KL}},
author={Yihan Cao and Yanbin Kang},
booktitle={The Fourth Blogpost Track at ICLR 2025},
year={2025},
url={https://openreview.net/forum?id=jGVCs8gomF}
}

@inproceedings{schweighofer2025on,
title={On Information-Theoretic Measures of Predictive Uncertainty},
author={Kajetan Schweighofer and Lukas Aichberger and Mykyta Ielanskyi and Sepp Hochreiter},
booktitle={Uncertainty in Artificial Intelligence},
year={2025},
url={https://openreview.net/forum?id=hLv7yZ268J}
}

@article{gneiting2007strictly,
  title={Strictly proper scoring rules, prediction, and estimation},
  author={Gneiting, Tilmann and Raftery, Adrian E},
  journal={Journal of the American statistical Association},
  volume={102},
  number={477},
  pages={359--378},
  year={2007},
  publisher={Taylor \& Francis}
}

@article{bregman1967relaxation,
  title={The relaxation method of finding the common point of convex sets and its application to the solution of problems in convex programming},
  author={Bregman, Lev M},
  journal={USSR computational mathematics and mathematical physics},
  volume={7},
  number={3},
  pages={200--217},
  year={1967},
  publisher={Elsevier}
}

@book{bach2024learning,
  title = {Learning theory from first principles},
  author = {Bach, Francis},
  publisher = {MIT Press},
  year = {2024}
}

@article{huang2021quantifying,
  title={Quantifying epistemic uncertainty in deep learning},
  author={Huang, Ziyi and Lam, Henry and Zhang, Haofeng},
  journal={arXiv preprint arXiv:2110.12122},
  year={2021}
}

@InProceedings{pmlr-v216-wimmer23a,
  title = 	 {Quantifying aleatoric and epistemic uncertainty in machine learning: Are conditional entropy and mutual information appropriate measures?},
  author =       {Wimmer, Lisa and Sale, Yusuf and Hofman, Paul and Bischl, Bernd and H\"ullermeier, Eyke},
  booktitle = 	 {Uncertainty in Artificial Intelligence},
  year = 	 {2023},
  url = 	 {https://proceedings.mlr.press/v216/wimmer23a.html}
}

@inproceedings{Balaji2017,
 author = {Lakshminarayanan, Balaji and Pritzel, Alexander and Blundell, Charles},
 booktitle = {Neural Information Processing Systems},
 title = {Simple and Scalable Predictive Uncertainty Estimation using Deep Ensembles},
 year = {2017}
}

@misc{huang2023,
      title={Efficient Uncertainty Quantification and Reduction for Over-Parameterized Neural Networks}, 
      author={Ziyi Huang and Henry Lam and Haofeng Zhang},
      year={2023},
      eprint={2306.05674},
      archivePrefix={arXiv},
      primaryClass={stat.ML},
      url={https://arxiv.org/abs/2306.05674}, 
}

@misc{wizgall2026groundtruthframeworkuncertaintydisentanglement,
      title={A Ground-Truth Framework for Uncertainty Disentanglement with Posterior Risk}, 
      author={Frieder Wizgall and Georg Tirpitz and Moritz Seiler and Kerstin Ritter and Bálint Mucsányi},
      year={2026},
      eprint={2608.05995},
      archivePrefix={arXiv},
      primaryClass={cs.LG},
      url={https://arxiv.org/abs/2608.05995}, 
}

@article{Banerjee2005,
author = {Banerjee, Arindam and Merugu, Srujana and Dhillon, Inderjit S. and Ghosh, Joydeep},
title = {Clustering with Bregman Divergences},
year = {2005},
issue_date = {12/1/2005},
publisher = {JMLR.org},
volume = {6},
issn = {1532-4435},
journal = {J. Mach. Learn. Res.},
month = dec,
pages = {1705–1749},
numpages = {45}
}

@ARTICLE{Nielsen2009,
  author={Nielsen, Frank and Nock, Richard},
  journal={IEEE Transactions on Information Theory}, 
  title={Sided and Symmetrized Bregman Centroids}, 
  year={2009},
  volume={55},
  number={6},
  pages={2882-2904},
  doi={10.1109/TIT.2009.2018176}}

@article{
gupta2022ensembles,
title={Ensembles of Classifiers: a Bias-Variance Perspective},
author={Neha Gupta and Jamie Smith and Ben Adlam and Zelda E Mariet},
journal={Transactions on Machine Learning Research},
issn={2835-8856},
year={2022},
url={https://openreview.net/forum?id=lIOQFVncY9},
note={}
}

@InProceedings{salelabel,
  title = 	 {Label-wise Aleatoric and Epistemic Uncertainty Quantification},
  author =       {Sale, Yusuf and Hofman, Paul and L\"ohr, Timo and Wimmer, Lisa and Nagler, Thomas and H\"ullermeier, Eyke},
  booktitle = 	 {Uncertainty in Artificial Intelligence},
  year = 	 {2024},
  url = 	 {https://proceedings.mlr.press/v244/sale24a.html}
}

@inproceedings{jimenez_2026position,
  title = {Position: Epistemic Uncertainty Estimation Methods are Fundamentally Incomplete},
  author = {Sebasti{\'a}n Jim{\'e}nez and Mira J{\"u}rgens and Willem Waegeman},
  booktitle = {Forty-third International Conference on Machine Learning Position Paper Track},
  year = {2026},
  url = {https://openreview.net/forum?id=g598HZM6ib}
}

@article{hofman2026uncertainty, title={Uncertainty Quantification for Machine Learning: One Size Does Not Fit All}, volume={40}, url={https://ojs.aaai.org/index.php/AAAI/article/view/39323}, DOI={10.1609/aaai.v40i26.39323}, number={26}, journal={Proceedings of the AAAI Conference on Artificial Intelligence}, author={Hofman, Paul and Sale, Yusuf and Hüllermeier, Eyke}, year={2026}, month={Mar.}, pages={21726–21734} }

@misc{pfau2025,
      title={A Generalized Bias-Variance Decomposition for Bregman Divergences}, 
      author={David Pfau},
      year={2025},
      eprint={2511.08789},
      archivePrefix={arXiv},
      primaryClass={cs.LG},
      url={https://arxiv.org/abs/2511.08789}, 
}

@book{savagebook,
  title = {{The Foundations of Statistics}},
  author = {Savage, Leonard J.},
  publisher = {John Wiley and Sons},
  year = {1954}
}

@incollection{ramsey1926truth,
  title={Truth and probability},
  author={Ramsey, Frank P},
  booktitle={Readings in formal epistemology: Sourcebook},
  pages={21--45},
  year={1926},
  publisher={Springer}
}

@InProceedings{bulte2026,
  title = 	 {An Axiomatic Assessment of Entropy- and Variance-based Uncertainty Quantification in Regression},
  author =       {B\"{u}lte, Christopher and Sale, Yusuf and L\"{o}hr, Timo and Hofman, Paul and Kutyniok, Gitta and H\"{u}llermeier, Eyke},
  booktitle = 	 {Uncertainty in Artificial Intelligence},
  year = 	 {2026},
  url = 	 {https://proceedings.mlr.press/v337/bulte26b.html}
}

@misc{proper_regression,
      title={Uncertainty Quantification for Regression using Proper Scoring Rules}, 
      author={Alexander Fishkov and Kajetan Schweighofer and Mykyta Ielanskyi and Nikita Kotelevskii and Mohsen Guizani and Maxim Panov},
      year={2025},
      eprint={2509.26610},
      archivePrefix={arXiv},
      primaryClass={cs.LG},
      url={https://arxiv.org/abs/2509.26610}, 
}

@InProceedings{smith2024rethinking,
  title = 	 {Rethinking Aleatoric and Epistemic Uncertainty},
  author =       {Bickford Smith, Freddie and Kossen, Jannik and Trollope, Eleanor and Van Der Wilk, Mark and Foster, Adam and Rainforth, Tom},
  booktitle = 	 {International Conference on Machine Learning},
  year = 	 {2025},
  url = 	 {https://proceedings.mlr.press/v267/bickford-smith25a.html}
}

@inproceedings{
kotelevskii2025from,
title={From Risk to Uncertainty: Generating Predictive Uncertainty Measures via Bayesian Estimation},
author={Nikita Kotelevskii and Vladimir Kondratyev and Martin Tak{\'a}{\v{c}} and Eric Moulines and Maxim Panov},
booktitle={The Thirteenth International Conference on Learning Representations},
year={2025},
url={https://openreview.net/forum?id=cWfpt2t37q}
}

@inproceedings{ agarwal2024onpolicy,
title={On-Policy Distillation of Language Models: Learning from Self-Generated Mistakes},
author={Rishabh Agarwal and Nino Vieillard and Yongchao Zhou and Piotr Stanczyk and Sabela Ramos Garea and Matthieu Geist and Olivier Bachem},
booktitle={The Twelfth International Conference on Learning Representations},
year={2024},
url={https://openreview.net/forum?id=3zKtaqxLhW}
}

@inproceedings{gu2024minillm,
title={Mini{LLM}: Knowledge Distillation of Large Language Models},
author={Yuxian Gu and Li Dong and Furu Wei and Minlie Huang},
booktitle={International Conference on Learning Representations},
year={2024},
url={https://openreview.net/forum?id=5h0qf7IBZZ}
}

@inproceedings{kendall2017,
author = {Kendall, Alex and Gal, Yarin},
title = {What uncertainties do we need in Bayesian deep learning for computer vision?},
year = {2017},
booktitle = {Neural Information Processing Systems}
}

@inproceedings{Kotelevskii2022,
author = {Kotelevskii, Nikita and Artemenkov, Aleksandr and Fedyanin, Kirill and Noskov, Fedor and Fishkov, Alexander and Shelmanov, Artem and Vazhentsev, Artem and Petiushko, Aleksandr and Panov, Maxim},
title = {Nonparametric uncertainty quantification for single deterministic neural network},
year = {2022},
isbn = {9781713871088},
publisher = {Curran Associates Inc.},
address = {Red Hook, NY, USA},
booktitle = {Proceedings of the 36th International Conference on Neural Information Processing Systems},
articleno = {2631},
numpages = {16},
location = {New Orleans, LA, USA},
series = {NIPS '22}
}

@article{gneiting2011making,
    title   = {Making and Evaluating Point Forecasts},
   author  = {Gneiting, Tilmann},
   journal = {Journal of the American Statistical Association},
   volume  = {106}, number = {494}, pages = {746--762}, year = {2011}}

@inproceedings{guo2017calibration,
  title     = {On Calibration of Modern Neural Networks},
  author    = {Guo, Chuan and Pleiss, Geoff and Sun, Yu and Weinberger, Kilian Q.},
  booktitle = {Proceedings of the 34th International Conference on Machine Learning},
  series    = {Proceedings of Machine Learning Research},
  volume    = {70},
  pages     = {1321--1330},
  year      = {2017},
  publisher = {PMLR}
}

@article{degroot1962,
  author  = {DeGroot, Morris H.},
  title   = {Uncertainty, Information, and Sequential Experiments},
  journal = {The Annals of Mathematical Statistics},
  volume  = {33},
  number  = {2},
  pages   = {404--419},
  year    = {1962},
  doi     = {10.1214/aoms/1177704567}
}

\newpage

\appendix

\newpage

\section{Proofs for Section~\ref{sec:GDE}}
\label{app:sec4proofs}

Throughout, we fix $\vx$ and use the assumptions of Section~\ref{sec:background}: $\ell$ is strictly proper with Savage representation, $\Gamma$ is linear, and all expectations are finite. Gradients of $\phi$ are only taken at points in $\operatorname{ri}(\gA)$, namely $p$ and $\bar q$. We write results in full-distribution form; the action form follows by replacing each $q$ with $a_q=\Gamma(q)$.
The proofs use the Bregman three-point identity \citep{Nielsen2009},
\begin{equation}
  B_\phi(u,w)
  =B_\phi(u,v)+B_\phi(v,w)
   +\big\langle\nabla\phi(v)-\nabla\phi(w),\,u-v\big\rangle.
  \label{eq:app-3pt}
\end{equation}
When $u$ is random with mean $v$, the inner product vanishes in expectation.

\subsection{Proof of \autoref{thm:gde_decomposition}}
\label{app:proof_gde}

\begin{proof}{\bf of \autoref{thm:gde_decomposition}.}
Applying \autoref{eq:subjective_Bregman} with belief $q_\Theta$ and report $p$, and taking the expectation over $\Theta$,
\begin{equation}
  \E_\Theta\big[\subjectiveriskqtheta{p}\big]
  =\E_\Theta\big[B_\phi(q_\Theta,p)\big]+\E_\Theta\big[H_\ell(q_\Theta)\big].
  \label{eq:app-gde-expected}
\end{equation}
Now apply \autoref{eq:app-3pt} with $u=q_\Theta$, $v=\bar q$ and $w=p$. Taking $\E_\Theta$, the inner product vanishes since $\E_\Theta[q_\Theta-\bar q]=0$, giving
\begin{equation}
  \E_\Theta\big[B_\phi(q_\Theta,p)\big]
  =\textcolor{blue}{B_\phi(\bar q,p)}
  +\textcolor{BrightOlive}{\E_\Theta\big[B_\phi(q_\Theta,\bar q)\big]},
  \label{eq:app-gde-split}
\end{equation}
which is the `reverse-argument' bias--variance decomposition of \citet{pfau2025}. Substituting \autoref{eq:app-gde-split} into \autoref{eq:app-gde-expected} gives the stated decomposition.
Note that $\bar q$ is a mixture, which lies in $\gP$ by convexity even when $\bar q\notin\gH$.

For log-loss on categorical distributions, $\phi(q)=\sum_y q(y)\ln q(y)$, so $B_\phi(q_\Theta,\bar q)=D_{\mathrm{KL}}(q_\Theta\Vert\bar q)$ and $H_\ell(q_\Theta)=H(q_\Theta)$. With $\hat Y\mid\Theta\sim q_\Theta$ and marginal $\hat Y\sim\bar q$,
\[
\E_\Theta\big[D_{\mathrm{KL}}(q_\Theta\Vert\bar q)\big]
= H(\bar q)-\E_\Theta\big[H(q_\Theta)\big]
=H(\hat Y)-H(\hat Y\mid\Theta)
=I(\hat Y;\Theta),
\]
and the entropy term is $\E_\Theta[H(q_\Theta)]=H(\hat Y\mid\Theta)$.
\end{proof}

\subsection{Proof of \autoref{thm:bregman_ltv}}
\label{app:proof_ltv}
Generalisations of the law of total variance to Bregman divergences exist in the literature. In clustering, \citet[Thm.~1]{Banerjee2005} split the Bregman information of a random variable into within-cluster and between-cluster terms. \citet[Lemma~3.2]{gupta2022ensembles}\footnote{Lemma~3.10 in arXiv:2206.10566. Their variance is centred on the dual mean; it is equivalent to ours by Legendre duality.} give a generalised law of total variance for arbitrary Bregman divergences. \citet[Prop.~3.8]{gruber2023properBiasVariance} extend it to functional Bregman divergences, and hence to strictly proper scores. We refer the reader to these, and give the proof in our notation for completeness.

\begin{proof}
Write $\bar q_{\scriptscriptstyle A\mid D_n}=\E_{A\mid D_n}[q_\Theta]$, so that $\E_{D_n}[\bar q_{\scriptscriptstyle A\mid D_n}]=\bar q$. Apply \autoref{eq:app-3pt} with $u=q_\Theta$, $v=\bar q_{\scriptscriptstyle A\mid D_n}$ and $w=\bar q$. Taking $\E_{A\mid D_n}$, the inner product vanishes since $\E_{A\mid D_n}[q_\Theta-\bar q_{\scriptscriptstyle A\mid D_n}]=0$, giving
\begin{equation}
  \E_{A\mid D_n}\big[B_\phi(q_\Theta,\bar q)\big]
  =\E_{A\mid D_n}\big[B_\phi(q_\Theta,\bar q_{\scriptscriptstyle A\mid D_n})\big]
  +B_\phi(\bar q_{\scriptscriptstyle A\mid D_n},\bar q).
  \label{eq:app-ltv-conditional}
\end{equation}
Taking $\E_{D_n}$ of \autoref{eq:app-ltv-conditional}, and noting $\E_{D_n}\E_{A\mid D_n}=\E_\Theta$, proves \autoref{eq:bregman_ltv_dp}.
\end{proof}
For squared loss, it is the classical law of total variance.
\subsection{Proof of \autoref{thm:ent_var}}
\label{app:proof_ent_var}
The key step is to treat the realised outcome itself as an action, $t(y)=\Gamma(\delta_y)$, as in \autoref{subsec:ent_var}. Under mild conditions, the {\em loss itself} is then a Bregman divergence from the realised outcome to the reported action. This is a classical result of \citet{savage1971elicitation} for expectations and for probabilities on finite outcome spaces \citep[see][Thm.~7]{gneiting2011making}, extended to vector-valued outcomes by \citet{banerjee2005optimality}. For completeness, we state it in our notation.
\begin{corollary}[Bregman loss functions]\label{cor:bregmanLoss}
Under the assumptions of \autoref{thm:gde_decomposition}, suppose
that the action space regret representation extends to every
point mass belief $\delta_y$, with
$t(y)=\Gamma(\delta_y)\in\gA$. Assume further that
\begin{enumerate}[label=(\roman*)]
    \item $\phi\big(t(y)\big)$ is finite for all $y\in\gY$;
    \item reporting the realised outcome yields zero loss, $\ell(t(y),y)=0$ for all $y\in\gY$\footnote{This is condition (S0) of \citet{gneiting2011making}.}.
\end{enumerate}
Then the loss is a Bregman divergence from the realised outcome to the reported action:
\begin{equation}
    \ell(a,y) = B_\phi\big(t(y),a\big) \qquad \text{for all } a\in\operatorname{ri}(\gA),\ y\in\gY.
    \label{eq:bregmanLoss}
\end{equation}
\end{corollary}
\begin{proof}
Evaluating \autoref{eq:subjective_Bregman} at the belief $q=\delta_y$, whose subjective risk is the loss itself, gives $\ell(a,y)=B_\phi(t(y),a)+\ell(t(y),y)$, which is \autoref{eq:bregmanLoss} under (ii).
\end{proof}

\begin{proof}{\bf of \autoref{thm:ent_var}.}
Under the assumptions of \autoref{thm:ent_var}, Corollary~\ref{cor:bregmanLoss} gives $\ell(a,y)=B_\phi(t(y),a)$.
Note that every belief $q$ is simply a mixture of point masses, $q=\E_{\hat Y\sim q}[\delta_{\hat Y}]$. By the linearity of $\Gamma$, the elicited action is the mean of the realised outcome, $a_q=\E_{\hat Y\sim q}[t(\hat Y)]$, and the generalised entropy becomes
\[
  H_\ell(q)=\subjectiverisk{a_q}=\E_{\hat Y\sim q}\Big[B_\phi\big(t(\hat Y),a_q\big)\Big],
\]
the expected divergence of the realised outcome from its mean, i.e.\ a {\em variance} (\autoref{eq:entropyIsVariance}).
Now apply \autoref{eq:app-3pt} with $u=t(\hat Y)$, $v=a_\Theta$ and $w=\bar a$. Taking $\E_Z$, the inner product vanishes since $\E_{\hat Y\mid\Theta}[t(\hat Y)]=a_\Theta$, which gives \autoref{eq:ltv_outcome}. By the variance form of $H_\ell$ above, the total variance $\E_Z\big[B_\phi(t(\hat Y),\bar a)\big]$ is $H_\ell(\bar q)$, the within-group variance $\E_\Theta\big[\E_{\hat Y\mid\Theta}[B_\phi(t(\hat Y),a_\Theta)]\big]$ is $\E_\Theta[H_\ell(q_\Theta)]$, and the between-group variance $\E_\Theta\big[B_\phi(a_\Theta,\bar a)\big]$ is the predictive disagreement. Hence
\[
  H_\ell(\bar q)=\E_\Theta\big[B_\phi(q_\Theta,\bar q)\big]+\E_\Theta\big[H_\ell(q_\Theta)\big],
\]
which is \autoref{eq:gde_as_ltv}.
\end{proof}

\newpage

\makeatletter
\newcommand{\proofsheading}[2]{%
  \@ifundefined{r@#1}{#2}{Proofs for Section~\ref{#1}: \nameref{#1}}}
\makeatother

\section{Proofs for Section~\ref{sec:SLT}} %
\label{app:slt_proofs}
This appendix collects the proofs of results in Section~\ref{sec:SLT}, in the order in which they appear.

\subsection{\proofsheading{subsec:subjgap}{Proofs for the Subjective Risk Gap}}

\begin{proof}{\bf of Theorem~\ref{thm:subjgap}.}
Setting $z=p$ in \autoref{eq:subjective_Bregman} and integrating over $X\sim p(\vx)$ gives $R_q(p)=\E_X\big[B_\phi(q_X,p_X)+H_\ell(q_X)\big]$. Subtracting the Bayes risk $R_p(p)=\E_X\big[H_\ell(p_X)\big]$ completes the proof.
\end{proof}

\begin{proof}{\bf of Proposition~\ref{prop:zero-gap}.}
By~\eqref{eq:gap_gradient}, $\nabla H_\ell(p_X)=0$ gives $\Delta_{\mathrm{subj}}(q)=0$ for every $q$. 
\end{proof}

\begin{proof}{\bf of Proposition~\ref{prop:diagnose}.}
By Theorem~\ref{thm:subjgap}, $U(q)-U(p)=\Delta_{\mathrm{subj}}(q)-G(q)$. Each case follows by subtracting
$G(q)\ge0$ from the stated range of $\Delta_{\mathrm{subj}}(q)$. The three ranges cover all values.
\end{proof}

\vspace{-1cm}

\subsection{Proofs for Grounding and Uncertainty under Changes of Resource}
\begin{proof}{\bf of Theorem~\ref{thm:subjgaps}.}
Each of $q$, $q_n$, $q_*$ and $p$ satisfies $R_z(p)=G(z)+U(z)$. Substituting this into each difference of \autoref{eq:subj_decomp}, and using $G(p)=0$, gives the result.
\end{proof}

\begin{proof}{\bf of Corollary~\ref{cor:eur}.}
For a predictive distribution $z$, \citet{smith2024rethinking}
measure uncertainty by the expected loss of a Bayes-optimal
action under $z$, i.e.\ the generalised entropy $H_\ell(z)$
(Section~\ref{sec:background}). Their expected uncertainty reduction at $X$ under unlimited additional data is therefore
$
H_\ell(q_n(\cdot\mid X))
-
\E_{D^+}[
H_\ell(q_\infty(\cdot\mid X))
],
$
where $D^+$ denotes the additional data stream and $q_\infty$ is the model obtained from $d_n\cup D^+$ in the limit $m\to\infty$. In our setting, new data is drawn i.i.d.\ from $p(\vx,y)$ and the learner is consistent, so $q_\infty=q_*$, which is a fixed model. Thus we can drop the expectation, and the expression above equals
$
H_\ell(q_n(\cdot\mid X))
-
H_\ell(q_*(\cdot\mid X)).
$
Averaging over $X\sim p(\vx)$ gives
$U(q_n)-U(q_*)$.
\end{proof}

\begin{proof}{\bf of Corollary~\ref{cor:misspec}.}
If $p\in\gH$, then $p$ minimises $R_p$ over $\gH$, so every best-in-class model has zero excess risk. Under strict propriety for the full distribution, the best-in-class model is then reality itself, $q_*=p$, and so $U(q_*)=U(p)$. 
\end{proof}

\vspace{-1cm}

\subsection{Proofs for the GDE Decomposition under Changes of Resource}

\noindent{\bf Proof of Corollary~\ref{cor:gde_estimation}.}
Applying \autoref{eq:gde_dn} at $d_\infty$ gives
\begin{equation*}
    \E_{\Theta\mid d_{\infty}}\Big[  R_{q_{\infty}}(p) \Big]
    =
    \E_{X\sim p(\vx)}\Big[ \textcolor{blue}{B_\phi(\bar{q}_{\infty}, p)} + \textcolor{BrightOlive}{\E_{\Theta\mid d_{\infty}}\big[ B_\phi(q_{\infty}, \bar{q}_{\infty}) \big]} + \textcolor{red}{\E_{\Theta\mid d_{\infty}}\big[H_\ell(q_{\infty})\big]} ~\Big\vert~ X \Big].
\end{equation*}
Since $\Theta\mid d_\infty$ is a Dirac distribution on the parameter of $q_*$, we have $q_\infty=\bar q_\infty=q_*$. Hence $\textcolor{BrightOlive}{\E_{\Theta\mid d_{\infty}}\big[ B_\phi(q_{\infty}, \bar{q}_{\infty}) \big]}=0$ and $\textcolor{red}{\E_{\Theta\mid d_{\infty}}\big[H_\ell(q_{\infty})\big]}=H_\ell(q_*)$, so the left-hand side is $R_{q_*}(p)=\E_{X\sim p(\vx)}\big[B_\phi(q_*,p)+H_\ell(q_*)\mid X\big]$. Subtracting this from \autoref{eq:gde_dn} gives \autoref{eq:cor.1}. For \autoref{eq:cor.2}, \autoref{eq:gde_as_ltv} gives $\E_{\Theta\mid d_n}\big[B_\phi(q_n,\bar q_n)\big]+\E_{\Theta\mid d_n}\big[H_\ell(q_n)\big]=H_\ell(\bar q_n)$ at each $X$. Substituting into \autoref{eq:cor.1} and integrating over $X\sim p(\vx)$ yields $G(\bar q_n)-G(q_*)+U(\bar q_n)-U(q_*)$, which equals $R_{\bar q_n}(p)-R_{q_*}(p)$ since $R_z(p)=G(z)+U(z)$ for any $z$. The final claim follows as for Corollary~\ref{cor:eur}, with $q_n$ replaced by the predictive $\bar q_n$.
\hfill\BlackBox
\section{Details of the Experiments in Section~\ref{subsec:exp_misspec}}
\label{app:exp_misspec}

Inputs are drawn as $x_1,x_2\overset{\text{iid}}{\sim}\mathcal N(0,1)$, and labels as $Y\mid\vx\sim\mathrm{Bernoulli}\big(p(y=1\mid\vx)\big)$, where $p(y=1\mid\vx)=1/(1+e^{-f(\vx)})$ and $f(\vx)=3\sin 2x_1+2x_2-\tfrac12x_1x_2$.
The true decision boundary separates the inputs where $y=1$ is more likely from those where $y=0$ is more likely. On the boundary, both labels have probability $0.5$. Since $f$ is non-linear, this boundary is curved (\autoref{fig:exp_setup}).
All terms are computed under the log loss in nats, with $G(q)$, $U(q)-U(p)$ and $\Delta_{\mathrm{subj}}(q)$ as in \autoref{thm:subjgap}.

\begin{figure}[h]
\centering
\includegraphics[width=0.8\linewidth]{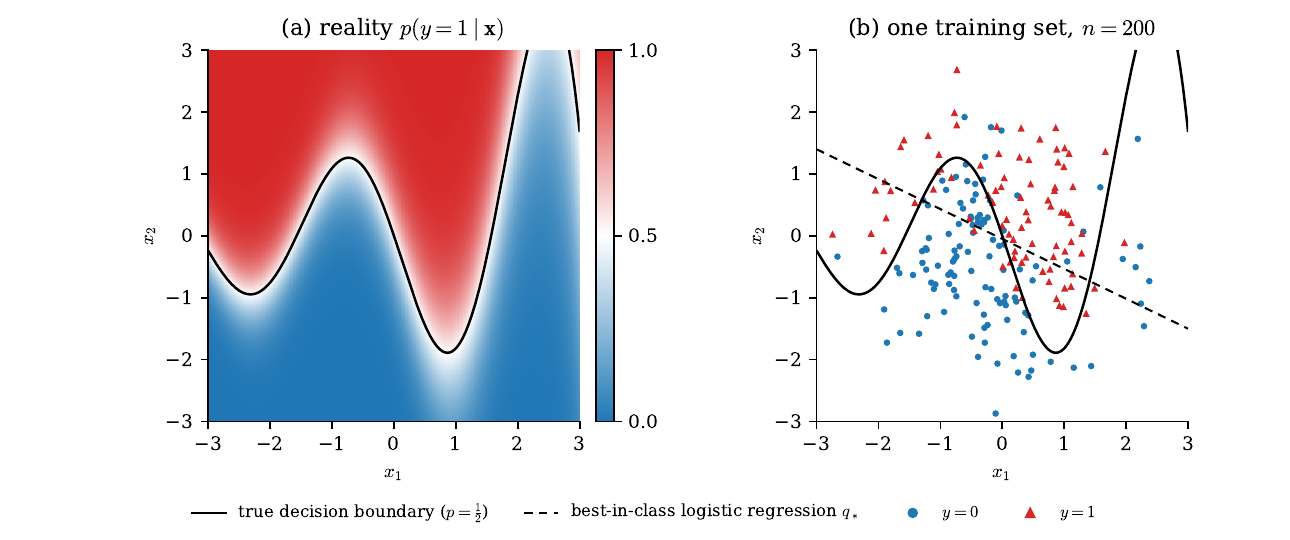}
\caption{The setting of Section~\ref{subsec:exp_misspec}. (a) Reality $p(y=1\mid\vx)$ and its decision boundary. (b) One training set ($n=200$), with the true boundary and that of the best-in-class logistic regression $q_*$.}
\label{fig:exp_setup}
\end{figure}
The well-specified class is a neural network with two hidden layers of width 32 and ReLU activations, trained with Adam until convergence. For each fit, we train five random initialisations and keep the one with the lowest training loss, which approximates the empirical risk minimiser $q_n$.
The misspecified class is logistic regression without regularisation, whose decision boundary is a straight line.
Both are trained to empirical risk minimisation, so that $q=q_n$ and the optimisation gap is negligible.
For logistic regression, the best-in-class model $q_*$ is computed exactly by minimising $R_p(q)$ directly.
For the neural network, $q_*$ is approximately $p$---trained on $2\times10^5$ samples, its grounding is below $10^{-3}$ nats.

For each $n$, we repeat the experiment 100 times. In each repeat, we draw a new training set. For the neural network we also use new random initialisations. Each repeat therefore gives a new trained model $\Theta=T(D_n,A)$.
\autoref{fig:exp_misspec} shows the mean over repeats, with a 95\% confidence interval.
In the second experiment, we use the same setting and vary the width of both hidden layers over $\{1,2,4,8,16,32,64\}$, training each network to convergence as above, for $n=200$ and $n=20{,}000$. Again, we repeat each experiment 100 times, and \autoref{fig:exp_capacity} shows the mean over repeats, with a 95\% confidence interval.

\autoref{fig:exp_models} shows the models of the second experiment trained on a single training set of $n=200$ samples.
With low capacity the models are under-confident and their boundaries are too simple. As capacity increases with this small training set, the models become over-confident and their boundaries overfit the training points.
Grounding and uncertainty are both needed for the full picture. For example, the models of width 2 and 64 have the same grounding, yet one is more uncertain than reality and the other is less.
\begin{figure}[h]
\centering
\includegraphics[width=\linewidth]{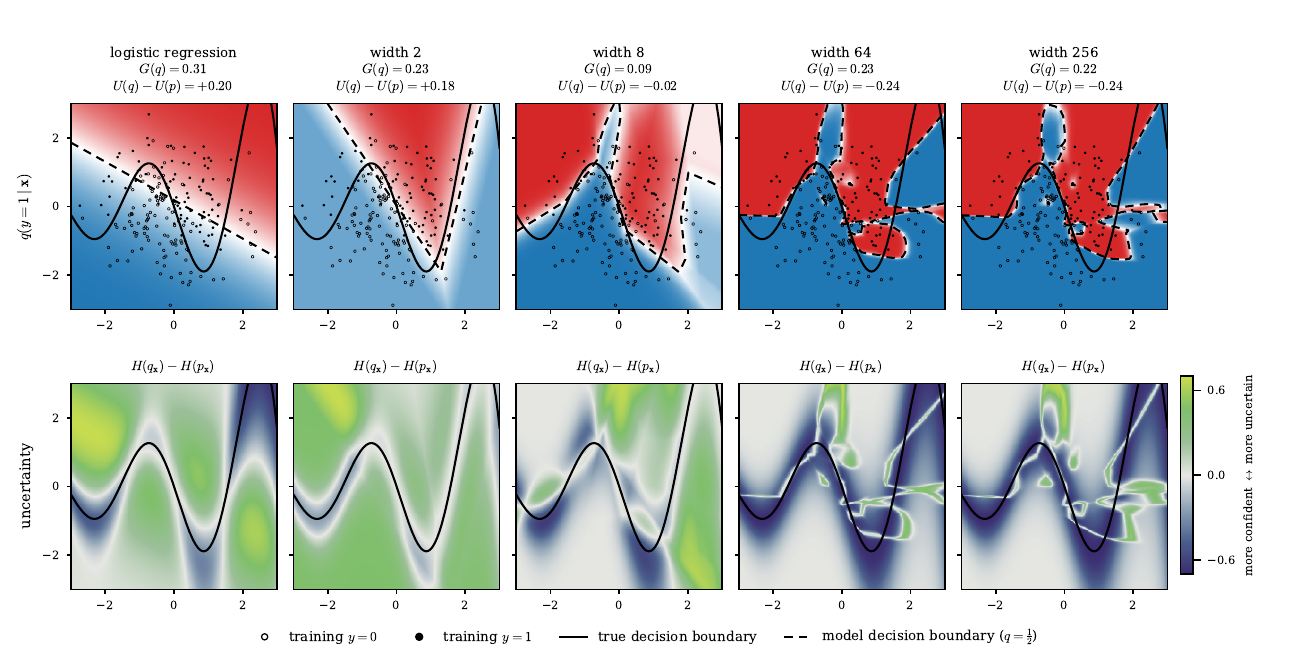}
\caption{Models of the second experiment, trained on $n=200$ samples. Top: each model's prediction $q(y=1\mid\vx)$ and its decision boundary, with its grounding $G(q)$ and uncertainty gap $U(q)-U(p)$. Bottom: the difference $H(q_\vx)-H(p_\vx)$ at each input.}
\label{fig:exp_models}
\end{figure}

\section{Proofs for Section~6}
\label{app:proofs-sec6}
\begin{proof}{\bf of Proposition~15}
Apply Equation~33 with $(u,v,w)=(q_\Theta,\bar q,\mathring q)$ for Equation~31,
and $(u,v,w)=(\bar q,\mathring q,q_\Theta)$ for Equation~32. Taking
$\mathbb{E}_\Theta$, the inner products vanish
since $\mathbb{E}_\Theta[q_\Theta]=\bar q$ and
$\mathbb{E}_\Theta[\ln q_\Theta]-\ln\mathring q$ is constant. If $\bar q$ has
mass outside the support of $\mathring q $, both sides are infinite.
\end{proof}

\end{document}